\documentclass[preprint,12pt]{elsarticle}

\usepackage[ruled,vlined]{algorithm2e}
\usepackage{url}
\usepackage{algorithmic} 
\usepackage{booktabs} 
\usepackage{amssymb}
\usepackage{amsthm}
\usepackage{lipsum}
\usepackage{multirow}
\usepackage[numbers]{natbib}  
\usepackage{amsmath}
\usepackage{natbib}  

\theoremstyle{definition}

\newtheorem{theorem}{Theorem}[section]

\journal{Nuclear Physics B}

\begin{document}

\begin{frontmatter}



\title{PUID: A Personalized Deconfounding Framework for Recommender Systems under Hidden Confounding} 

\author{Zongyu Li \\
    School of Computer Science, \\
    Guangdong University of Technology, \\
    Guangzhou, China \\
    gdutzongyuli@gmail.com}

\begin{abstract}

Recommender systems often rely on observational user–item interaction data, which is prone to selection bias due to users’ selective interactions with items. While techniques such as inverse propensity weighting (IPW) and doubly robust estimators are effective in addressing selection bias from observed confounding, they become unreliable when hidden confounding exists, meaning that there are confounders that influence both user clicks and feedback but are not observable (e.g., user salary). Existing approaches relying on randomized controlled trials (RCTs) or global sensitivity bounds are constrained in practice: RCTs demand costly experimental data, while global sensitivity bounds presume a uniformly bounded effect of unmeasured confounders on propensities through sensitivity analysis, thereby neglecting heterogeneity across user–item interactions. To overcome this limitation, we propose a novel framework, Personalized Unobserved-Confounding-aware Interaction Deconfounder (PUID), which estimates user–item level sensitivity bounds, thereby substantially relaxing the homogeneity assumption inherent in global sensitivity bounds. Under mild assumptions, we use an entropy-based method to estimate the individualized strength of hidden confounding. Specifically, if observed user and item features can already well predict the exposure status (i.e., the mutual information is small), then the influence of hidden confounding is assumed to be small. To ensure both robustness and predictive accuracy, we further develop an adversarial optimization strategy and propose a benchmark-guided variant (BPUID) that incorporates pre-trained models as stabilizing references. Extensive experiments on three real-world datasets demonstrate that our approach consistently outperforms state-of-the-art baselines under hidden confounding, without requiring RCT data.

\end{abstract}


\end{frontmatter}




\section{Introduction}

Recommender systems are widely used to deliver personalized content
across various applications, such as e-commerce platforms (e.g., Alibaba,
Amazon), social media (e.g., Facebook, Weibo), and lifestyle services (e.g.,
Yelp, Meituan)~\citep{chen2023bias}. A common characteristic of these systems is that users typically interact with only a small portion of available items. For example, in movie recommendation, users often choose to rate the movies they enjoyed, while ignoring those they disliked~\cite{zhang2019deep}. Similarly, advertising systems tend to display only those ads that are predicted to match user interests, while suppressing others. These selective exposure mechanisms, if not properly addressed, can introduce systematic bias.

Due to the high sparsity of user–item interactions in recommender systems, a substantial amount of feedback data is missing. Early approaches commonly adopted the Missing At Random (MAR) assumption, under which the probability of a rating being missing does not depend on its unobserved value. However, this assumption rarely holds in practice. In reality, missing data in recommender systems are often Missing Not At Random (MNAR): whether a user provides feedback often depends on how much they like or dislike the item~\cite{marlin2012collaborative}. For instance, users are more likely to rate items they enjoy. This means that the probability of observing a rating, known as the propensity, varies significantly across items and users. Recent studies have shown that explicitly modeling the MNAR nature of data can help improve the quality of recommended items~\cite{wang2018modeling}.

Given that ratings are Missing Not At Random (MNAR), calculating averages solely based on observed ratings can lead to substantial bias, potentially over- or under-estimating prediction errors by a large amount, commonly referred to as bias~\cite{little2019statistical}.  For example, in a movie recommendation system, users tend to rate only the movies they like, while rarely rating those they dislike~\citep{pradel2012ranking}. Similarly, ad placement systems selectively display advertisements believed to match user interests, while suppressing less relevant ones. Such selection mechanisms, if unadjusted, may introduce confounding bias by over-recommending popular items regardless of their actual relevance, a phenomenon known as selection bias~\cite{wei2021model}.

To correct for selection bias, several well-established strategies have been proposed. These include error imputation~\cite{chang2010training,hernandez2014probabilistic,steck2010training}, inverse propensity weighting (IPW)~\cite{imbens2015causal,luo2021unbiased,saito2020doubly,schnabel2016recommendations}, and doubly robust estimation~\cite{keohane2009counterfactuals,saito2020doubly,wang2019doubly}.

\begin{itemize}
    \item The error-imputation-based (EIB) approach estimates the prediction error for each missing rating and combines these imputed errors with observed errors to train or evaluate models ~\cite{steck2013evaluation}. However, this method often suffers from bias due to imputation inaccuracy~\cite{dudik2011doubly}, and such errors can accumulate during model training.
    
    \item The inverse propensity scoring (IPS) method, on the other hand, assigns higher importance to ratings that are less likely to be observed, by weighting observed errors inversely by their propensities~\cite{schnabel2016recommendations}, which refer to the probabilities of different ratings being observed and vary across user–item interactions. Although IPS can reduce bias in expectation, it tends to suffer from high variance: when propensities are very small, the inversely weights become large and unstable, potentially causing training losses to fluctuate and harming generalization performance~\cite{thomas2016data}.
    
    \item The doubly robust (DR) estimator combines the strengths of both approaches and achieves the key advantage of double robustness, meaning it remains consistent if either the imputation model or the propensity model is correctly specified. By integrating imputed errors and propensity weights in a unified framework, the DR estimator enables unbiased estimation while mitigating the instability caused by high variance in propensities~\cite{wang2019doubly}.
\end{itemize}

These approaches are effective under the assumption that all confounders are observed. However, they fail in the presence of unobserved confounders. These include factors such as user intent, financial status, or demographic characteristics that simultaneously influence both item exposure and user feedback. This issue, known as hidden confounding, is widespread in real-world recommendation scenarios~\cite{ding2022addressing,li2023removing,li2023balancing}. Therefore, training to unbiasedly evaluate recommendation models in the presence of hidden confounding has become both a necessary and challenging paradigm. When hidden confounders exist, the propensity score estimated using only observed confounders deviates from the true propensity score \cite{ding2022addressing}, formally expressed as: 
\[
P(o_{u, i}=1 \mid x_{u, i}) \neq P(o_{u, i}=1 \mid x_{u, i}, h_{u, i}),
\] where $x_{u, i}$ is the covariate of user $u$ and item $i$, $o_{u, i}$ indicates whether the interaction between user $u$ and item $i$ is observed, and $h_{u, i}$ represents hidden confounders.

To address hidden confounding bias in recommender systems, existing approaches can be broadly categorized into two classes: sensitivity analysis-based methods and data fusion-based methods.

Sensitivity analysis-based methods assume that the strength of hidden confounding can be bounded, such that the inverse of the true propensity scores lie within a range defined by the estimated propensity scores. These methods typically adopt worst-case optimization strategies to provide robustness against unmeasured confounding~\cite{ding2022addressing}. A representative example is the Robust Deconfounder (RD)~\cite{ding2022addressing}, which addresses the effect of unmeasured confounders on propensities under the mild assumption that such effects are bounded. RD estimates this bound via sensitivity analysis and incorporates it into an adversarial learning framework, thereby training a recommender model that remains robust to unmeasured confounding within the specified bound. Consequently, it may yield overly conservative or overly loose sensitivity bounds for certain users depending on factors such as their profile completeness.

Data fusion-based methods leverage auxiliary experimental data collected from randomized controlled trials (RCTs) or A/B tests to calibrate propensity and imputation models, thereby mitigating hidden confounding bias~~\cite{li2023balancing,xiao2024addressing}. For example, a balancing method has been proposed to leverage unbiased ratings~\cite{li2023balancing} obtained from RCTs to adjust for unobserved confounding and enhance estimation accuracy, although such approaches entail costly and sometimes ethically challenging data collection~\cite{jiang2024confounder,ma2022learning}.

To bridge above gaps, we propose a novel framework named Personalized Unobserved-Confounding-aware Interaction Deconfounder(PUID) that can quantify hidden confounding strength by estimating user-item-specific sensitivity
bounds. Specifically, we adopts a sensitivity analysis framework to estimate a personalized bound of the nominal propensity scores for each user-item interaction, based on observed data. Then, to infer confounding strength, we introduce an entropy-based assumption that quantify the unmeasured confounding contribution to observable entropy reduction. Morever, we design a benchmark-guided extension (BPUID) that incorporates a pre-trained model into learning as a stabilizing reference, thereby mitigating the robustness–accuracy trade-off commonly seen in worst-case optimization. The main contributions of this paper are summarized as follows:
\begin{itemize}
    \item To the best of our knowledge, this is the first work to model personalized hidden confounding strength via user-item-specific bounds, enabling personalized estimation of confounding strength.
    \item We develop an entropy-based sensitivity estimation method using discretized information gain, allowing the estimation of confounding strength from purely observational data.
    \item We design an adversarial training strategy and introduce a benchmark-guided extension (BPUID), which together ensure robustness and predictive performance under hidden confounding.
\end{itemize}

\section{Related Work}

\subsection{Classic Approaches to Selection Bias Correction}

To address selection bias in recommendation, prior work has developed three primary categories of estimators: error-imputation-based (EIB)~\cite{chang2010training,hernandez2014probabilistic,steck2010training}, inverse propensity scoring (IPS)~\cite{imbens2015causal,luo2021unbiased,saito2020doubly,schnabel2016recommendations}, and doubly robust (DR) methods~\cite{keohane2009counterfactuals,saito2020doubly,wang2019doubly}. The error-imputation-based (EIB) estimator approximates the prediction error for unobserved ratings and combines these imputed errors with observed errors to train models. While EIB is practical, it is prone to bias due to imputation inaccuracies, which can accumulate and propagate during training. The inverse propensity scoring (IPS) estimator addresses selection bias by weighting observed errors inversely proportional to their exposure propensities, defined as the probability of each rating being observed for a given user–item pair. Although IPS yields unbiased estimates in expectation under correctly specified propensities, it is susceptible to high variance when propensities are small, potentially destabilizing training and impairing generalization. The doubly robust (DR) estimator integrates both approaches, combining imputed errors with propensity weights, and achieves the property of double robustness: it remains consistent if either the imputation model or the propensity model is correctly specified, while mitigating the high-variance issues associated with IPS. Nevertheless, the presence of unmeasured confounders violates the underlying assumptions of IPS and DR, thereby undermining their theoretical guarantees and potentially inducing systematic bias in loss estimation.

\subsection{Data fusion-based Approaches}

Data fusion-based approaches aim to mitigate confounding by incorporating auxiliary data sources. For instance, KDRec~\cite{liu2020general} adopts a knowledge distillation framework where a teacher model is first trained on a small set of uniform data, and its predictions are then transferred to a student model trained on biased data. However, since uniform data is often limited in quantity and collected at the cost of user experience, the teacher model suffers from high variance, reducing its reliability. Moreover, the approach lacks theoretical justification, and how the knowledge transfer offsets bias remains unclear.

Several methods attempt to leverage a small amount of unbiased data to improve debiasing performance. Wang et al.~\cite{wang2021combating} propose a bi-level optimization framework, where a small unbiased subset is used to adaptively learn propensity weights for biased data. Similarly, AutoDebias~\cite{chen2021autodebias} applies meta-learning to optimize debiasing parameters by making use of uniform data, improving the performance of existing IPS and DR-based methods.

InterD~\cite{ding2022interpolative} introduces a knowledge distillation framework that integrates both biased and debiased models as teachers, and interpolates their predictions based on teacher confidence. Although it improves recommendation quality, it does not explicitly account for hidden confounding.

Only recently have some works begun to address hidden confounding in recommendation. Li et al.~\cite{li2023balancing} propose a model-agnostic balancing method that uses a few unbiased ratings to iteratively adjust model parameters learned from biased data and adaptively reweight samples, mitigating unobserved confounding and model misspecification. In a related effort, Xiao et al.~\cite{xiao2024addressing} present MetaDebias, which leverages heterogeneous observational datasets, in which only a portion of the data is affected by hidden confounding, and applies meta-learning to achieve unbiased learning.

\subsection{Sensitivity analysis-based Approaches}
Robust Deconfounder (RD)~\cite{ding2022addressing} is a sensitivity analysis-based method that accounts for the influence of unmeasured confounders on exposure propensities under the mild assumption that this influence is bounded. RD estimates this bound using sensitivity analysis and learns a recommender model robust to unmeasured confounding within the bound through adversarial learning. Existing implementations apply a single global bound for all user–item pairs, which may inadequately capture the heterogeneity of confounding effects. In practice, the magnitude of hidden confounding varies with the completeness of user and item attributes: entities with more complete features tend to experience minimal confounding, whereas those with sparse attributes are more susceptible to substantial confounding. Consequently, employing a uniform bound across all entities can lead to underestimation of confounding for sparse-attribute entities and overestimation for rich-attribute ones.

\section{Preliminaries}

\begin{table}[htbp]
\centering
\caption{Summary of Notations Used Throughout the Paper}
\label{tab:notation_summary}
\begin{tabular}{ll}
\toprule
\textbf{Notation} & \textbf{Description} \\
\midrule
$P$, $\mathbb{E}$ & Probability distribution and expectation \\
$\mathcal{D}$ & Set of all user-item pairs \\
$\mathcal{O}$ & Set of observed interactions: $\{(u,i) \mid o_{u,i}=1\}$ \\
$o_{u,i}$ & Binary exposure variable (1 if observed) \\
$r_{u,i}$ & True rating or outcome for user $u$ on item $i$ \\
$x_{u,i}$ & Observed features for interaction $(u,i)$ \\
$h_{u,i}$ & Unobserved confounding for interaction $(u,i)$ \\
$x_u, x_i$ & User-specific and item-specific features \\
$\hat{r}_{u,i}$ & Predicted rating by model \\
$e_{u,i}$ & Prediction error: $|\hat{r}_{u,i} - r_{u,i}|$ or squared error \\
$\hat{e}_{u,i}$ & Estimated oracle error via imputation model \\
$\phi$, $\theta$ & Parameters of prediction model and imputation model \\
$\hat{\phi}^{(0)}$, $\hat{\theta}^{(0)}$ & Benchmark estimates from base model \\
\midrule
$p_{u,i}$ & Nominal propensity: $P(o_{u,i}=1 \mid x_{u,i})$ \\
$\tilde{p}_{u,i}$ & True propensity: $P(o_{u,i}=1 \mid x_{u,i}, h_{u,i})$ \\
$\hat{p}_{u,i}$ & Estimated nominal propensity score \\
$\tilde{w}_{u,i}$ & Inverse of true propensity: $1/\tilde{p}_{u,i}$ \\
\midrule
$H(X)$ & Shannon entropy of random variable $X$ \\
$H(Y \mid X)$ & Conditional entropy of $Y$ given $X$ \\
$\mathrm{IG}(Y; X)$ & Information gain: $H(Y) - H(Y \mid X)$ \\
$\mathrm{IG}_h$ & Information gain attributed to hidden confounding \\
\midrule
$\Gamma$ & Global sensitivity parameter \\
$\Gamma_{u,i}$ & Personalized sensitivity parameter for interaction $(u,i)$ \\
$a_{u,i}$, $b_{u,i}$ & Lower and upper bounds on $\tilde{w}_{u,i}$ under global $\Gamma$ \\
$\tilde{a}_{u,i}$, $\tilde{b}_{u,i}$ & Personalized bounds on $\tilde{w}_{u,i}$ under $\Gamma_{u,i}$ \\
\midrule
$\tilde{W}$ & Global uncertainty set of inverse weights \\
$\mathcal{W}_\text{P}$ & Personalized uncertainty set with bounds $[\tilde{a}_{u,i}, \tilde{b}_{u,i}]$ \\
\midrule
$\mathcal{L}_{\text{RD-IPS}}$ & Robust IPS objective under RD framework \\
$\mathcal{L}_{\text{RD-DR}}$ & Robust DR objective under RD \\
$\mathcal{L}_{\text{Imp}}^{\text{RD-DR}}$ & Imputation loss under RD-DR \\
$\mathcal{L}_{\text{PUID-IPS}}$ & Personalized robust IPS objective \\
$\mathcal{L}_{\text{PUID-DR}}$ & Personalized robust DR objective \\
$\mathcal{L}_{\text{Imp}}^{\text{PUID-DR}}$ & Personalized robust imputation objective \\
$\mathcal{L}_{\text{BPUID-IPS}}$ & Benchmarked PUID-IPS objective \\
$\mathcal{L}_{\text{BPUID-DR}}$ & Benchmarked PUID-DR objective \\
$\mathcal{L}_{\text{Imp}}^{\text{BPUID-DR}}$ & Benchmarked PUID-DR imputation loss \\
\bottomrule
\end{tabular}
\end{table}

\subsection{Problem Setting and Notation}

Table~\ref{tab:notation_summary} outlines the key notations employed throughout this paper.

Let $\mathcal{U}$ and $\mathcal{I}$ denote the complete sets of users and items, respectively. We consider $\mathcal{U}_o = \{u_1, \ldots, u_m\} \subseteq \mathcal{U}$ and $\mathcal{I}_o = \{i_1, \ldots, i_n\} \subseteq \mathcal{I}$ as the observed subsets of users and items, which are randomly drawn from the underlying superpopulation. The user-item interaction space is then defined as $\mathcal{D} = \mathcal{U}_o \times \mathcal{I}_o = \{(u, i) \mid u \in \mathcal{U}_o, i \in \mathcal{I}_o\}$. For each user-item pair $(u, i) \in \mathcal{D}$, let $x_{u,i} \in \mathbb{R}^K$ denote the associated feature vector, $r_{u,i} \in \mathbb{R}$ the true (ground-truth) rating, and $o_{u,i} \in \{0,1\}$ a binary indicator that specifies whether $r_{u,i}$ is observed. The set of all observed interactions is defined as
\[
\mathcal{O} = \{(u, i) \in \mathcal{D} \mid o_{u,i} = 1 \}.
\]Let $\hat{r}_{u,i} = f(x_{u,i}; \theta_r)$ denote the prediction model, parameterized by $\theta_r$, which estimates the true rating $r_{u,i}$ based on the feature vector $x_{u,i}$. The objective is to minimize the ideal loss computed over the complete user-item interaction space:

\begin{equation}
\mathcal{L}_{\text{Ideal}}(\phi) = \frac{1}{|\mathcal{D}|} \sum_{(u,i) \in \mathcal{D}} e_{u,i},
\label{eq:ideal}
\end{equation}

\begin{equation}
e_{u,i} = \mathcal{L}(\hat{r}_{u,i}, r_{u,i}),
\label{eq:eui}
\end{equation}
where $\mathcal{L}(\cdot, \cdot)$ is a loss function such as mean squared error or cross-entropy loss.

In practice, however, the true rating $r_{u,i}$ is only partially observed, making $e_{u,i}$ inaccessible for instances where $o_{u,i} = 0$. A common workaround is to minimize the \textit{naive loss}, which considers only the observed interactions:

\begin{equation}
\mathcal{L}_{\text{Naive}}(\phi) = \frac{1}{|\mathcal{O}|} \sum_{(u,i) \in \mathcal{O}} e_{u,i}.
\label{eq:naive}
\end{equation}

Due to selection bias, this estimator deviates from the true ideal loss and yields a biased estimate~\cite{wang2019doubly}. To address this issue, many previous studies incorporate the propensity score $p_{u,i}$ to reweight the observed samples, resulting in the Inverse Propensity Scoring (IPS) loss:

\begin{equation}
\mathcal{L}_{\text{IPS}}(\phi) = \frac{1}{|\mathcal{D}|} \sum_{(u,i) \in \mathcal{D}} \frac{o_{u,i} \cdot e_{u,i}}{\hat{p}_{u,i}},
\label{eq:ips}
\end{equation}where $\hat{p}_{u,i} = \pi(x_{u,i}; \theta_p)$ denotes the estimated propensity score. The IPS estimator remains unbiased provided that $\hat{p}_{u,i}$ exactly matches the true propensity $\hat{p}_{u,i}$~\cite{schnabel2016recommendations, wang2019doubly}.

To further enhance robustness and mitigate the high variance issue of IPS, the Doubly Robust (DR) estimator has been proposed~\cite{wang2019doubly, wang2021combating, li2023balancing}, formulated as:

\begin{equation}
\mathcal{L}_{\text{DR}}(\phi,\theta) = \frac{1}{|\mathcal{D}|} \sum_{(u,i) \in \mathcal{D}} \left[ \hat{e}_{u,i} + \frac{o_{u,i} \cdot (e_{u,i} - \hat{e}_{u,i})}{\hat{p}_{u,i}} \right],
\label{eq:dr}
\end{equation}where $\hat{e}_{u,i} = m(x_{u,i}; \theta_e)$ denotes the imputed error from an auxiliary model. The DR estimator remains unbiased if either the estimated propensity scores or imputed errors are accurate. Thus, accurate estimation of propensity scores is essential for effective debiasing in both IPS and DR-based approaches.

\subsection{Selection Bias and Hidden Confounding}

Prior studies generally operate under the no hidden confounding assumption, which posits that the exposure indicator $o_{u,i}$ is conditionally independent of the potential outcome $r_{u,i}$ given the observed confounders $x_{u,i}$. In practice, however, this assumption may be violated due to the presence of unobserved confounders, such as user salary. Such violations are formalized in the second part of Equation~\ref{eq:oxh}, where $o_{u,i}$ is no longer conditionally independent of $r_{u,i}$ given only the observed confounders $x_{u,i}$.

\begin{equation}
o_{u, i} \perp r_{u, i} \,\big|\, \left(x_{u, i}, h_{u, i} \right),
\quad
o_{u, i} \not\perp r_{u, i} \,\big|\, x_{u, i},
\label{eq:oxh}
\end{equation}where the symbols $\perp$ and $\not \perp$ represent independent and not independent.

\begin{theorem}
(Propensity under hidden Confounding~\cite{ding2022addressing})
\label{thm:hidden_confounding}
In the presence of hidden confounding $h$, the standard estimators suffer from inherent bias:
(a) Both the IPS and DR estimators become biased even if $\hat{p}_{u, i}$ and $\hat{e}_{u, i}$ accurately estimate the nominal propensity $p_{u,i}$ and the imputed error $g_{u,i}$, respectively.
(b) Under this setting, the true propensity score can be defined as:

\begin{equation}
\tilde{p}_{u, i} = P\left( o_{u, i} = 1 \mid x_{u, i}, h_{u, i} \right),
\label{eq:poxh}
\end{equation}and assume that $\hat{p}_{u, i}$ accurately estimates $\tilde{p}_{u, i}$, then both the IPS and DR estimators remain unbiased. 
\end{theorem}

Theorem~\ref{thm:hidden_confounding} underscores that the specification of the propensity score is pivotal for ensuring the unbiasedness of both IPS and DR estimators. In practice, estimating the propensity score using only the measured confounding $x$ merely control for the confounding bias introduced by $x$, leaving the bias induced by the hidden confounding $h$ unaddressed. Achieving complete removal of confounding bias requires incorporating both the observed confounding $x$ and the unobserved confounding $h$ into the propensity estimation. However, accurately estimating the full propensity score $\tilde{p}_{u, i}$ is infeasible in the presence of completely missing $h$, unless one imposes strong assumptions.

\subsection{Information-Theoretic Foundations}

Information-theoretic quantities offer a rigorous way to characterize uncertainty and dependence in probabilistic systems. In our framework, we use entropy-based metrics to capture how both observed and hidden confounders influence exposure probability.

\noindent \textbf{Entropy.} For a discrete random variable $X$ with distribution $\{p_i\}$, the Shannon entropy is defined as:
\begin{equation}
H(X) = - \sum_{i} p_i \log p_i,
\end{equation}
measuring the average uncertainty in $X$.

\noindent \textbf{Conditional Entropy.} Given two variables $X$ and $Y$, the conditional entropy
\begin{equation}
H(Y \mid X) = - \sum_{x,y} P(x,y) \log P(y \mid x)
\end{equation}
quantifies the remaining uncertainty in $Y$ after observing $X$.

\noindent \textbf{Information Gain.} The reduction in uncertainty about $Y$ due to knowledge of $X$ is given by:
\begin{equation}
\mathrm{IG}(Y; X) = H(Y) - H(Y \mid X).
\end{equation}
This reflects how informative $X$ is about $Y$.

\section{Proposed Methods}

\subsection{Sensitivity Analysis and Motivation for Personalization}

Sensitivity analysis is a classical technique used to assess the robustness of causal inference in the presence of hidden confounding~\cite{ding2022addressing}. In the context of recommendation, it provides a principled way to reason about the deviation between the true exposure probability $\tilde{p}_{u,i}$ and the estimated nominal propensity $p_{u,i} = P(o_{u,i}=1 \mid x_{u,i})$.

\citet{ding2022addressing} typically adopts a logistic model to express the nominal propensity score:
\begin{equation}
p_{u,i} = \frac{\exp(m(x_{u,i}))}{1 + \exp(m(x_{u,i}))},
\end{equation}where $m(x_{u,i})$ is an arbitrary function of observed features.

To account for hidden confounding, an additive model is proposed to express the true propensity $\tilde{p}_{u,i}$:
\begin{equation}
\tilde{p}_{u,i} = P(o_{u,i}=1 \mid x_{u,i}, h_{u,i}) 
= \frac{\exp(m(x_{u,i}) + \varphi(h_{u,i}))}{1 + \exp(m(x_{u,i}) + \varphi(h_{u,i}))},
\end{equation}where $\varphi(h_{u,i})$ is an unknown scalar capturing the effect of the unobserved confounder.

It is assumed that the strength of hidden confounding is bounded, i.e., $|\varphi(h_{u,i})| \leq \log \Gamma$, for some $\Gamma \geq 1$. Under this constraint, the following odds ratio inequality can be derived:
\begin{equation}
\frac{1}{\Gamma} \leq \frac{(1 - p_{u,i}) \cdot \tilde{p}_{u,i}}{p_{u,i} \cdot (1 - \tilde{p}_{u,i})} \leq \Gamma,
\end{equation}

This further implies that the inverse of the true propensity score, 
$\tilde{w}_{u,i} = 1 / \tilde{p}_{u,i}$, is bounded within a interval:
\begin{equation}
a_{u,i} \leq \tilde{w}_{u,i} \leq b_{u,i}, 
\label{eq:bound_w}
\end{equation}
where the lower and upper bounds are given by
\begin{equation}
a_{u,i} = 1 + ({1/p_{u,i} - 1})/{\Gamma}, 
\qquad
b_{u,i} = 1 + \big(1/p_{u,i} - 1 \big) \cdot \Gamma,
\label{eq:bound_ab}
\end{equation}
and $p_{u,i}$ denotes the nominal propensity estimated from observed data, 
while $\Gamma$ controls the global sensitivity to hidden confounding.

A key limitation of existing RD methods lies in their reliance on a global sensitivity parameter $\Gamma$ shared across all user-item interactions. Although RD improves robustness under hidden confounding, it still relies on a global sensitivity parameter $\Gamma$ shared across all user-item interactions. In practice, exposure behavior is highly heterogeneous. In practice, the extent of hidden confounding varies with the completeness of user and item attributes. Users or items with more complete attribute information tend to exhibit minimal hidden confounding, whereas those with incomplete attributes are more susceptible to substantial hidden confounding. Consequently, applying a uniform sensitivity parameter across all users or items may lead to under-correction for entities with sparse attributes and over-correction for those with rich attributes.

This limitation motivates the introduction of personalized sensitivity bounds $\Gamma_{u,i}$, which account for heterogeneity in the magnitude of hidden confounding across different user–item pairs. In the subsequent section, we describe an entropy-based approach to infer these bounds directly from observational data.

\subsection{Personalized Unobserved-Confounding-aware Interaction Deconfounder Framework}

We propose the \textbf{Personalized Unobserved-Confounding-aware Interaction Deconfounder (PUID)} framework, which addresses heterogeneity in hidden confounding and leverages feature information for robust recommendation. PUID consists of two components:

    \begin{itemize}
    \item \textbf{Personalized Sensitivity Bound Estimation}: We infer user-item-specific bounds $\Gamma_{u,i}$ by leveraging entropy to quantify the residual information in the exposure $o_{u,i}$ attributable to unobserved confounders. Under the assumption that hidden confounders of the same type contribute proportionally to observed features, these entropy-based constraints enable principled estimation of personalized bounds from observational data.

    \item \textbf{Feature-Aware Representation Enhancement}: User and item features are mapped to latent embeddings via multi-layer perceptrons (MLPs) and integrated into a matrix factorization (MF) framework. This MF-MLPs design captures richer context and ensures architectural consistency across baseline comparisons.
\end{itemize}

\subsubsection{Personalized Sensitivity Bound Estimation}

Entropy provides a principled tool in information theory for quantifying uncertainty. In the context of exposure modeling, the uncertainty of $o_{u,i}$ can be progressively reduced when conditioning on different types of variables. Both observed features (e.g., $x_u$, $x_i$) and unobserved confounders (e.g., $h_u$, $h_i$) contribute information, thereby decreasing the entropy of $o_{u,i}$.

We posit that variables of the same category exert comparable influences on exposure entropy. Specifically, observed user features $x_u$ and hidden user-level confounders $h_u$ are both user-side variables. Thus, the entropy reduction provided by $h_u$ should be proportional to that provided by $x_u$. Analogously, on the item side, the reduction induced by hidden confounders $h_i$ should be proportional to that induced by observed features $x_i$.

This proportionality motivates the following entropy-difference constraints, which capture the residual information gain attributable to hidden confounders:

\begin{equation}
\frac{H(o_{u,i}) - H(o_{u,i} \mid x_u)}{H(o_{u,i} \mid x_u, x_i) - H(o_{u,i} \mid x_u, x_i, h_u)} > 0,
\end{equation}

\begin{equation}
\frac{H(o_{u,i} \mid x_u) - H(o_{u,i} \mid x_u, x_i)}{H(o_{u,i} \mid x_u, x_i, h_u) - H(o_{u,i} \mid x_u, x_i, h_u, h_i)} > 0.
\end{equation}These ratio-based formulations formalize the assumption that hidden confounders contribute systematically to entropy reduction in a manner consistent with their observed counterparts. In this way, the effect of unobserved confounding can be interpreted as the unexplained portion of exposure informativeness, aligning with the properties of conditional mutual information.

To make this influence operationally measurable, we introduce a surrogate metric of hidden confounding impact as a weighted combination of observable information gains:
\begin{equation}
\mathrm{IG}_h \propto \alpha \cdot \left[ H(o_{u,i}) - H(o_{u,i} \mid x_u) \right] 
+ \beta \cdot \left[ H(o_{u,i} \mid x_u) - H(o_{u,i} \mid x_u, x_i) \right],
\end{equation}
where $\alpha, \beta \geq 0$ are hyperparameters that control the relative importance of user- and item-side observables. 
Here, the symbol ``$\propto$'' indicates that $\mathrm{IG}_h$ is defined up to a positive scaling factor, 
i.e., we only care about its relative magnitude rather than its absolute value. 

The central contribution of this work is the precise quantification of hidden confounding bias and the enablement of personalized sensitivity estimation through analyzing the distinct effects of observed and hidden confounders on exposure uncertainty. Specifically, when observed variables substantially reduce exposure uncertainty, the explanatory power of hidden confounders correspondingly diminishes. Conversely, limited information from observed variables implies a stronger influence of hidden confounding. Crucially, this proportional relationship between information contributions not only provides a rigorous theoretical foundation but also facilitates data-driven, user–item level sensitivity estimation. More importantly, our method enables the inference and quantification of hidden confounding strength solely from observed data, which provides a flexible way to mitigate the unmeasured confounding and avoid the dangers of relying on unrealistic assumptions.

\noindent \textbf{Discretized Estimation.} 
Conditional entropies are estimated empirically using an Adaptive Hybrid Binning Strategy. 
For binary exposure $o_{u,i}$ and any feature set $x$, the conditional entropy is approximated by:
\begin{equation}
\hat{H}(o_{u,i} \mid x) 
= - \sum_{j=1}^k P(x \in \text{bin}_j) 
\sum_{o \in \{0,1\}} P(o \mid \text{bin}_j) 
\log P(o \mid \text{bin}_j),
\label{eq:Houi_x}
\end{equation}where $k$ denotes the number of quantile bins. Each $\text{bin}_j$ represents a partition of the feature space, constructed by adaptively segmenting $x$ according to its empirical distribution to ensure adequate sample coverage.
For multivariate $x$, binning is performed jointly over the entire discrete feature space by grouping feature combinations into composite bins.
Each bin thus reflects either a specific configuration of discrete feature values or a merged set of such configurations, depending on their empirical frequency.

\noindent \textbf{Constructing Personalized Sensitivity Bounds.} To quantify the impact of hidden confounding at the instance level, we leverage the estimated sensitivity score $\Gamma_{u,i}$ to derive a personalized interval for the inverse true propensity:

\begin{equation}
\tilde{a}_{u,i} = 1 + (1/p_{u,i} - 1)/  \Gamma_{u,i}, \quad 
\tilde{b}_{u,i} = 1 + (1/p_{u,i} - 1) \cdot \Gamma_{u,i},
\label{eq:aui_bui}
\end{equation}where $p_{u,i}$ is the nominal (observed) propensity score, and $\Gamma_{u,i}$ is a personalized sensitivity estimate defined as:

\begin{equation}
\Gamma_{u,i} =\alpha \cdot \left[ \hat{H}(o_{u,i}) - \hat{H}(o_{u,i} \mid x_u) \right] 
+ \beta \cdot \left[ \hat{H}(o_{u,i} \mid x_u) - \hat{H}(o_{u,i} \mid x_u, x_i) \right],
\label{eq:gamma_ui}
\end{equation}where $\alpha$ and $\beta$ are non-negative weights controlling the relative importance of user- and item-side observed features.

This formulation enables the construction of a personalized uncertainty set:

\begin{equation}
\mathcal{W}_{\mathrm{P}} = \left\{ \tilde{W} \in R_{+}^{|\mathcal{D}|} \mid 
\tilde{a}_{u,i} \leq \tilde{w}_{u,i} \leq \tilde{b}_{u,i},\ \forall (u,i) \in \mathcal{D} \right\},
\label{eq:w_p}
\end{equation}where the global uncertainty set of inverse weights is denoted by $\tilde{W}$, and the personalized uncertainty set with bounds $[\tilde{a}_{u,i}, \tilde{b}_{u,i}]$ is represented by $\mathcal{W}_\text{P}$.

In contrast to conventional approaches that impose a uniform global sensitivity bound across all interactions, our method establishes user–item level bounds that adapt to the heterogeneous nature of confounding. This personalized formulation avoids the overly conservative or excessively loose constraints that often arise under a global assumption, thereby yielding a more faithful characterization of bias at the instance level. Moreover, by retaining the structural advantage of the global-bound setting, our framework naturally induces an uncertainty set over propensity scores. This uncertainty set provides a rigorous foundation for incorporating adversarial optimization techniques, enabling the development of robust propensity-based estimators.

\subsubsection{Feature-aware Matrix Factorization Backbone}
\label{sec:feature_backbone}

To better leverage auxiliary information, we extend the conventional matrix factorization (MF) framework by incorporating feature-aware representations for both users and items. In particular, we employ dedicated multi-layer perceptrons (MLPs) to transform user and item features into latent embeddings, which are subsequently integrated into the MF framework through elementwise interactions. We refer to this enhanced architecture as MF-MLPs for brevity. Additionally, to ensure fair and consistent comparisons, we apply the same feature-aware enhancement strategy to the MF components in baseline models, thereby aligning them with the MF-MLPs framework. This guarantees architectural consistency across experiments and enables a clearer evaluation of the effectiveness of feature integration.

Given user feature $\mathbf{x}_u \in R^{d_u}$ and item feature $\mathbf{x}_i \in R^{d_i}$, 
we adopt two-layer multi-layer perceptrons (MLPs) to project them into a shared latent space:
\begin{equation}
\mathbf{z}_u = \mathrm{MLP}_u(\mathbf{x}_u), \quad 
\mathbf{z}_i = \mathrm{MLP}_i(\mathbf{x}_i),
\label{eq:mlp}
\end{equation}
where each MLP consists of two fully connected layers with ReLU activation.

To account for personalized offsets, we introduce learnable user-item level scalar biases 
$b_u, b_i \in R$. The final prediction is given by:
\begin{equation}
\hat{r}_{ui} = \langle \mathbf{z}_u, \mathbf{z}_i \rangle + b_u + b_i,
\label{eq:prediction}
\end{equation}
where $\langle \cdot, \cdot \rangle$ denotes the dot product. This feature-aware MF backbone integrates collaborative signals with auxiliary side information, 
improving model expressiveness and personalization.

\subsection{Proposed Algorithms}
We propose two novel algorithms to address the confounding bias problem:

\begin{itemize}
    \item \textbf{Personalized Unobserved-Confounding-aware Interaction Deconfounder (PUID)}: An advanced framework that incorporates personalized sensitivity bounds into both Inverse Propensity Scoring (IPS) and Doubly Robust (DR) estimators through robust optimization objectives.
    
    \item \textbf{Benchmark-guided PUID (BPUID)}: An enhanced variant of PUID that introduces benchmark-guided regularization to stabilize the training process and optimally balance the robustness-accuracy trade-off.
\end{itemize}

The subsequent subsections elaborate on the technical details of each algorithm component.
\subsubsection{Robust Training with Personalized Bounds}

With the personalized sensitivity bounds $\Gamma_{u,i}$ and the derived interval $[a_{u,i}, b_{u,i}]$, we construct a robust learning objective that explicitly accounts for hidden confounding by constraining the inverse propensity weights within a calibrated uncertainty range. This formulation enables the construction of a personalized uncertainty set as defined in Equation~\ref{eq:w_p}.

\vspace{0.5em}
\noindent\textbf{PUID-IPS Objective.}
Let $e_{u,i} = |\hat{r}_{u,i} - r_{u,i}|$ or $(\hat{r}_{u,i} - r_{u,i})^2$ be the prediction error. 
The personalized robust IPS objective is defined as:
\begin{equation}
\mathcal{L}_{\text{PUID-IPS}}(\phi) 
= \max_{\tilde{W} \in \mathcal{W}_{\text{P}}} \frac{1}{|\mathcal{D}|} 
\sum_{(u,i) \in \mathcal{D}} o_{u,i} \cdot e_{u,i} \cdot \tilde{w}_{u,i},
\label{eq:puid-ips}
\end{equation}where $\tilde{W} = \{\tilde{w}_{u,i}\}_{(u,i)\in\mathcal{D}}$ denotes the vector of inverse-propensity weights for all user-item pairs, $\mathcal{W}_{\text{P}}$ is the personalized uncertainty set for all user-item pairs, $\tilde{w}_{u,i} \in [\tilde{a}_{u,i}, \tilde{b}_{u,i}]$ is the weight assigned to a specific user-item pair $(u,i)$.

Intuitively, $\tilde{W}$ represents the collection of all candidate weight assignments, 
$\mathcal{W}_{\text{P}}$ constrains them within the estimated personalized bounds, 
and $\tilde{w}_{u,i}$ is the individual component applied in the summation of Eq.~\eqref{eq:puid-ips}.

\begin{algorithm}[t]
\caption{Personalized Unobserved-Confounding-aware DR (PUID-DR)}
\label{alg:puid-dr}
\KwIn{Observed data $\mathcal{D}$, nominal propensities $p_{u,i}$}
\KwOut{Optimized recommendation model}
\BlankLine
Initialize recommendation model parameters $\phi$\;
Initialize imputation model parameters $\theta$\;
Estimate entropy-based sensitivity $\Gamma_{u,i}$ using Eq.~\eqref{eq:gamma_ui}\;
Compute personalized inverse propensity interval $[\tilde{a}_{u,i}, \tilde{b}_{u,i}]$ via Eq.~\eqref{eq:aui_bui}\;
Construct personalized uncertainty set $\mathcal{W}_P$ from Eq.~\eqref{eq:w_p}\;
\While{stopping condition not met}{
    Sample $(u,i)$ from $\mathcal{D}$\;
    Compute prediction error $e_{u,i}$ using current model $\phi$\;
    Estimate $\hat{e}_{u,i} = g_\theta(x_{u,i})$ using imputation model $\theta$\;
    Maximize Eq.~\eqref{eq:puid-ips} over $\tilde{W} \in \mathcal{W}_P$ to update base model weights\;
    Minimize Eq.~\eqref{eq:puid-ips} w.r.t.\ $\phi$ to optimize base model\;
    Maximize Eq.~\eqref{eq:puid-dr} over $\tilde{W} \in \mathcal{W}_P$ to update imputation model weights\;
    Minimize Eq.~\eqref{eq:puid-dr} w.r.t.\ $\theta$ to optimize imputation model\;
}
\Return{Optimized recommendation model with parameters $\phi$}\;
\end{algorithm}

\begin{algorithm}[t]
\caption{Benchmarked Personalized Unobserved-Confounding-aware DR (BPUID-DR)}
\label{alg:bpuid-dr}
\KwIn{Observed data $\mathcal{D}$, nominal propensities $p_{u,i}$}
\KwOut{Optimized recommendation model}
\BlankLine
Use $\mathcal{D}$ and $p_{u,i}$ to train a baseline DR model by minimizing Eq.~\eqref{eq:poxh}, and fix its parameters $\hat{\phi}^{(0)}$ and $\hat{\theta}^{(0)}$ as the benchmark\;
Initialize recommendation model parameters $\phi$ and $\theta$\;
Estimate personalized sensitivity $\Gamma_{u,i}$ using Eq.~\eqref{eq:gamma_ui}\;
Compute inverse propensity interval $[\tilde{a}_{u,i}, \tilde{b}_{u,i}]$ via Eq.~\eqref{eq:aui_bui}\;
Construct uncertainty set $\mathcal{W}_P$ as in Eq.~\eqref{eq:w_p}\;
\While{stopping condition not reached}{
    Sample $(u,i)$ from $\mathcal{D}$\;
    Compute prediction error $e_{u,i}(\phi)$ based on current model $\phi$ and Eq.~\eqref{eq:eui}\;
    Compute benchmark error $e_{u,i}(\hat{\phi}^{(0)})$ using the fixed DR model and Eq.~\eqref{eq:eui}\;
    Maximize Eq.~\eqref{eq:bpuid-dr} over $\tilde{W} \in \mathcal{W}_P$ to update base model weights\;
    Minimize Eq.~\eqref{eq:bpuid-dr} w.r.t.\ $\phi$ to optimize base model\;
    Maximize Eq.~\eqref{eq:bpuid-dr-imp} over $\tilde{W} \in \mathcal{W}_P$ to update imputation model weights\;
    Minimize Eq.~\eqref{eq:bpuid-dr-imp} w.r.t.\ $\theta$ to optimize imputation model\;
}
\Return{Optimized recommendation model with parameters $\phi$} and $\theta$\;
\end{algorithm}

\vspace{0.5em}
\noindent\textbf{PUID-DR Objective.}
Given an imputation model that estimates the oracle error $\hat{e}_{u,i}$, the personalized robust DR estimator minimizes:
\begin{equation}
\mathcal{L}_{\text{PUID-DR}}(\phi, \theta) 
= \max_{\tilde{W} \in \mathcal{W}_{\text{P}}} \frac{1}{|\mathcal{D}|} 
\sum_{(u,i) \in \mathcal{D}} \left[ 
\hat{e}_{u,i} + o_{u,i} \cdot (e_{u,i} - \hat{e}_{u,i}) \cdot \tilde{w}_{u,i} 
\right].
\label{eq:puid-dr}
\end{equation}

This formulation combines the robustness of inverse weighting and the flexibility of error imputation, while the constraint set $\mathcal{W}_\text{P}$ ensures adaptive robustness at the interaction level. The training procedure of PUID-DR is outlined in Algorithm~\ref{alg:puid-dr}, where a personalized uncertainty set is constructed for robust inverse propensity weighting.

\vspace{0.5em}
\noindent\textbf{Imputation Model Objective.}
To train the imputation model, we adopt a robust error-matching objective that minimizes the discrepancy between $\hat{e}_{u,i}$ and $e_{u,i}$ under worst-case weighting:

\begin{equation}
\mathcal{L}_{\text{Imp}}^{\text{PUID-DR}}(\theta) 
= \max_{\tilde{W} \in \mathcal{W}_{\text{P}}} \frac{1}{|O|} 
\sum_{(u,i) \in O} \left( \hat{e}_{u,i} - e_{u,i} \right)^2 \cdot \tilde{w}_{u,i}.
\label{eq:puid-dr-imp}
\end{equation}

We adopt an alternating optimization strategy, where the prediction model $\phi$ and imputation model $\theta$ are updated in turn. The parameters $\phi$ and $\theta$ are optimized in an alternating manner. Given fixed $\hat{\theta}$, the prediction model $\phi$ is updated by minimizing the robust doubly robust objective Eq.~\eqref{eq:puid-dr}.
\noindent Conversely, given fixed $\hat{\phi}$, the imputation model $\theta$ is updated by minimizing the adversarial imputation loss Eq.~\eqref{eq:puid-dr-imp}.

The PUID-DR method introduces an adversarial procedure for the imputation model. Directly relying on $\hat{p}_{u,i}$ may induce bias in the imputation process and increase the risk of performance degradation. By incorporating adversarial optimization, the procedure mitigates instability arising from unmeasured confounders. Intuitively, the proposed PUID-IPS/DR approaches safeguard against worst-case scenarios induced by hidden confounding, thereby achieving greater robustness than non-adversarial estimators.

In practice, handling unmeasured confounding remains highly challenging, as the functional relationships among $h$, $o$, and $r$ are often complex and unknown. Rather than attempting to completely eliminate hidden confounders, the PUID framework calibrates the loss function through uncertainty sets derived from sensitivity analysis in causal inference. This design provides a principled and flexible means of mitigating the effects of hidden confounding while avoiding dependence on unrealistic assumptions.

To see why personalized bounds ensure robustness, we provide the following explanation. The interval $[a_{u,i}, b_{u,i}]$ is derived through sensitivity analysis rooted in causal inference, where $\Gamma_{u,i}$ quantifies the plausible deviation of the true propensity from the nominal value due to unmeasured confounding. By optimizing over this personalized interval, we incorporate an adversarial perturbation mechanism that guards against worst-case shifts in the exposure mechanism while avoiding the need to assume full knowledge of hidden confounders $h_{u,i}$.

Furthermore, unlike global bounds that impose a uniform uncertainty set across all instances, often leading to over-correction for entities with rich attributes and under-correction for those with sparse attributes, our personalized formulation adapts to the attribute completeness of each $(u,i)$ pair. This allows for tighter bounds for users or items with more complete attributes and looser bounds for those with incomplete attributes, achieving a better trade-off between robustness and informativeness.

In the next subsection, we extend this robust formulation with a benchmark-guided extension strategy to further stabilize training and improve predictive performance.

\subsubsection{Benchmarked PUID Framework}

The proposed PUID framework enhances the robustness of existing propensity-based methods by accounting for hidden confounding. However, there is currently no theoretical guarantee that such robustness improvements will translate into better prediction accuracy. To address this limitation, we introduce the Benchmarked PUID (BPUID) framework, which ensures prediction accuracy improvement by anchoring the optimization to a benchmark estimator.

For clarity, we denote the prediction error as $e_{u,i}(\phi)$, explicitly indicating its dependence on the model parameter $\phi$. Suppose we have access to an initial estimator of $\phi$, denoted by $\hat{\phi}^{(0)}$, obtained from a baseline method. Then, the corresponding BPUID-IPS estimator can be expressed as:

\begin{equation}
\mathcal{L}_{\text{BPUID-IPS}}(\phi) 
= \max_{\tilde{W} \in \mathcal{W_P}} \frac{1}{|\mathcal{D}|} 
\sum_{(u, i) \in \mathcal{D}} o_{u, i} 
\left\{ e_{u, i}(\phi) - e_{u, i}\left(\hat{\phi}^{(0)}\right) \right\} \tilde{w}_{u, i}.
\label{eq:bpuid-ips}
\end{equation}

The BPUID-IPS estimator differs fundamentally from the PUID-IPS estimator. In contrast to Eq.\eqref{eq:puid-ips}, the introduction of a benchmark, as formalized in Eq.\eqref{eq:bpuid-ips}, alters the solution for $w_{u,i}$ and consequently calibrates the estimation of the model parameters $\phi$. In a similar manner, the BPUID-DR estimator and the corresponding imputation model, can be formulated as follows:

\begin{align}
\mathcal{L}_{\text{BPUID-DR}}(\phi, \theta) 
&= \max_{\tilde{W} \in \mathcal{W}_{\text{P}}} \frac{1}{|\mathcal{D}|} \sum_{(u, i) \in \mathcal{D}} 
\left\{ \left[ \hat{e}_{u, i}(\theta) - \hat{e}_{u, i}\left(\hat{\theta}^{(0)}\right) \right] \right. \nonumber \\
&\quad + \left[ o_{u, i} \left( e_{u, i}(\phi) - \hat{e}_{u, i}(\theta) \right) \right. \nonumber \\
&\quad - \left. \left. o_{u, i} \left( e_{u, i}\left(\hat{\phi}^{(0)}\right) - \hat{e}_{u, i}\left(\hat{\theta}^{(0)}\right) \right) \right] \tilde{w}_{u, i} \right\},  
\label{eq:bpuid-dr}
\\[1ex] 
\mathcal{L}_{\text{Imp}}^{\text{BPUID-DR}}(\phi, \theta) 
&= \max_{\tilde{W} \in \mathcal{W}_{\text{P}}} \frac{1}{|O|} \sum_{(u, i) \in \mathcal{D}} 
\left\{ \left[ \hat{e}_{u, i}(\theta) - e_{u, i}(\phi) \right]^2 \right. \nonumber \\
&\quad - \left. \left[ \hat{e}_{u, i}\left(\hat{\theta}^{(0)}\right) - e_{u, i}\left(\hat{\phi}^{(0)}\right) \right]^2 \right\} \tilde{w}_{u, i},
\label{eq:bpuid-dr-imp}
\end{align}
where $\hat{\phi}^{(0)}$ and $\hat{\theta}^{(0)}$ are the benchmark estimates of $\phi$ and $\theta$. The training procedure of BPUID-DR is outlined in Algorithm~\ref{alg:bpuid-dr}, where a personalized uncertainty set is constructed for robust inverse propensity weighting.

The proposed BPUID-IPS/BPUID-DR approaches are very flexible due to the free choice of the benchmark estimators. In addition, note that $\mathcal{L}_{BPUID-IPS}(\phi)=0$ if $\phi=\hat{\phi}^{(0)}$, which implies that the minimum value of $\mathcal{L}_{BPUID-I P S}(\phi)$ is non-positive. Similarly, it is not hard to derive that $\min _\phi \mathcal{L}_{BPUID-DR}(\phi) \leq 0$. Thus, the BPUID-IPS/BPUID-DR estimators achieve smaller losses than the corresponding benchmark
estimators.

\section{Experiment}
\label{experiment}

We conduct experiments to answer the following research questions:

\begin{itemize}
  \item \textbf{RQ1}: Do the proposed PUID and BPUID methods outperform existing SOTA methods under conditions of hidden confounding?
  \item \textbf{RQ2}: Do our methods stably perform well under scenarios with different levels of hidden confounding?
  \item \textbf{RQ3}: Is there a consistent directional alignment and a potential proportional relationship between the information gain from observed and unobserved confounding in the exposure mechanism?
  \item \textbf{RQ4}: Does the information gain increase monotonically with the strength of hidden confounding?
\end{itemize}

\subsection{Experimental Settings}

\textbf{Dataset.} To evaluate the effectiveness of PUID and BPUID, we conduct experiments on four public datasets from different application domains: 
1) \textit{KuaiRec}\footnote{\url{https://github.com/chongminggao/KuaiRec}}, 
2) \textit{Coat}\footnote{\url{https://www.cs.cornell.edu/~schnabts/mnar/}}, and 
3)
\textit{Yahoo!R3}\footnote{\url{https://webscope.sandbox.yahoo.com/}}. 
These datasets cover various domains including micro-video recommendation (KuaiRec), e-commerce (Coat), and music recommendation (Yahoo!R3). 

\textit{KuaiRec} provides both a biased interaction matrix collected from real-world platform logs and an unbiased subset obtained via randomized exposure, containing 1,325,197 biased and 270,180 unbiased interactions over 7,176 users and 10,729 items. 
\textit{Coat} and \textit{Yahoo!R3} similarly offer both biased and unbiased splits, with 6,960/4,640 and 311,704/54,000 interactions respectively.

For data preprocessing, we binarize explicit ratings to create positive/negative feedback labels. Specifically, for KuaiRec, ratings greater than 2 are labeled as positive ($1$), and those rated 2 or below as negative ($-1$). For Coat and Yahoo!R3, scores above 3 are considered positive ($1$), and those rated 2 or below as negative ($-1$). These binarization rules follow common practice for unbiased recommendation evaluation.

\subsubsection{Feature Selection and Discretization for KuaiRec}

To improve the stability of the model and reduce potential feature redundancy, we perform systematic selection and discretization on the high-dimensional user and item behavior features in the KuaiRec dataset.

\textbf{Correlation-based Selection.} For item-side features, we calculate the Pearson correlation coefficient between each pair of features and set a threshold of 0.9 to identify groups of strongly correlated features. Within each group, only one representative feature is retained to avoid redundancy. For example, \texttt{valid\_play\_user\_num} is used to represent play-related behaviors. On the user side, only one pair of highly correlated features is identified: \{{onehot\_feat2, onehot\_feat3}\}, from which one is selected.

\textbf{Feature Discretization.} We employ three strategies: (1) categorical feature mapping (e.g., video\_type, visible\_status); (2) numerical feature binning via \texttt{qcut}, with fallback to binary or equal-width binning when necessary; (3) integer encoding for all features to support embedding-based models. A total of 19 item features and 14 user features are retained.

For more details on feature preprocessing, please refer to Appendix~\ref{appendix:feature_processing}.

\subsubsection{Feature Construction for ID-based Datasets}
For ID-based datasets such as \textit{Yahoo!R3} and \textit{Coat}, where no meaningful user or item side features are provided, we leverage matrix factorization (MF) to extract dense continuous representations as pseudo features. Specifically, we employ MF to generate 64-dimensional embeddings for both users and items on \textit{Yahoo!R3}, and 32-dimensional embeddings on \textit{Coat}, aligning with dataset scale and sparsity.

To facilitate entropy-based information gain analysis, we discretize these continuous representations via unsupervised clustering. For each individual feature dimension, we apply K-Means clustering to partition the continuous values into 4 categories. This quantization transforms latent factors into discrete, interpretable pseudo-categorical variables, enabling subsequent causal and statistical analysis under feature-based frameworks.

Such a feature construction pipeline bridges the gap between purely ID-based collaborative filtering signals and feature-driven modeling, ensuring fair comparability across different modeling paradigms.

\subsubsection{Biased Observation and Masked Subset Construction}

To simulate realistic exposure bias, we leverage the estimated propensity scores to downsample the unbiased interaction matrix (small\_matrix) through Bernoulli trials. Specifically, each user-item interaction is retained with a probability proportional to its feedback group’s estimated observation likelihood. This yields a synthesized biased training set that mirrors practical exposure patterns with inherent noise. For the KuaiRec dataset, this procedure generates 740,271 biased interactions from the original unbiased pool, covering 1,411 users and 3,327 items. To ensure consistency and compatibility with indexing-based models, user and item IDs are remapped into consecutive indices starting from zero.

To further evaluate model robustness under varying levels of data scarcity, we construct additional masked subsets by randomly removing 10\%, 20\%, 50\%, and 80\% of interactions from the biased dataset. These subsets simulate increasingly severe missing data scenarios, enabling assessment of model generalization under different sparsity regimes.

Meanwhile, the original unbiased interaction matrix is preserved as an unbiased baseline for both training and evaluation. This allows us to directly measure the model’s capacity to recover from confounding bias by contrasting its performance on clean versus biased data.

\subsubsection{Compared Methods}

To ensure fair and rigorous comparison, all methods share a unified feature-aware matrix factorization (MF-MLPs) backbone as our baseline model. This backbone consistently integrates both user/item IDs and their side information through dedicated multi-layer perceptrons (MLPs), enabling richer and more expressive user-item representations.

As detailed in Section~\ref{sec:feature_backbone}, our backbone enhances conventional MF 
by embedding user and item features via dedicated MLPs and computing predictions as in 
Eq.~\ref{eq:mlp} and Eq.~\ref{eq:prediction}. 
This unified backbone is applied across all methods, ensuring architectural consistency 
and eliminating potential confounding effects from feature utilization discrepancies.

We apply the proposed \textbf{PUID} and \textbf{BPUID} to three fundamental propensity-based models—Inverse Propensity Score (IPS)~\cite{imbens2015causal,luo2021unbiased,saito2020doubly,schnabel2016recommendations} and Doubly Robust (DR)~\cite{keohane2009counterfactuals,saito2020doubly,wang2019doubly}—to validate their effectiveness. We further compare them with three additional confounding-aware frameworks:

\textbf{Base model}~\cite{koren2009matrix}. We employ widely used a unified feature-aware matrix factorization (MF-MLPs) as the base model for propensity-based methods, and adopt
the public implement to train it with biased data.

\textbf{IPS}~\cite{schnabel2016recommendations}. IPS is a classical propensity-based method. It estimates the propensity, and uses the inverse propensity score to weight the loss function.

\textbf{DR}~\cite{wang2019doubly}. DR is another classical propensity-based method. It trains a base model and a imputation model with loss functions adjusted by the propensity score.

\textbf{RD}~\cite{ding2022addressing}. The Robust Deconfounder framework augments propensity-based estimation through integrated sensitivity analysis and bounded adversarial training, demonstrating general applicability across both Inverse Propensity Scoring and Doubly Robust methods while maintaining theoretical coherence.

\textbf{BRD-IPS}~\cite{ding2022addressing}. BRD-IPS incorporates the benchmark-regularized Robust Deconfounder (BRD) architecture into the IPS estimator. This design leverages benchmark model guidance to mitigate robustness–accuracy trade-offs, offering more stable and generalizable performance.

\textbf{BRD-DR}~\cite{ding2022addressing}. BRD-DR adapts the BRD framework to the DR estimator, combining benchmark-guided regularization with the doubly robust structure. The approach benefits from improved robustness to hidden confounding while preserving the low-bias property of DR under correct model specification.

\textbf{Multi-IPS}~\cite{zhang2020large}. The Multi-IPS estimator combines inverse propensity scoring with multi-task learning to simultaneously mitigate selection bias and alleviate data sparsity in conversion prediction tasks. It approaches the problem from a causal perspective, explicitly modeling the missing-not-at-random nature of user interactions. By jointly learning multiple related tasks, Multi-IPS leverages shared information to alleviate the sparsity challenge, while employing inverse propensity weighting to correct for selection bias in the exposure–click–conversion pipeline.

\textbf{Multi-DR}~\cite{zhang2020large}. Multi-DR extends the Multi-IPS architecture by incorporating a doubly robust estimation scheme. In addition to propensity-based reweighting, it trains an auxiliary imputation model to estimate the counterfactual conversion outcomes for unclicked items. This combination reduces variance and improves robustness under model misspecification. Like Multi-IPS, Multi-DR is implemented within a multi-task learning framework to simultaneously mitigate data sparsity and correct for bias.

\textbf{PUID-IPS}. PUID-IPS applies our proposed Personalized Unobserved-Confounding-aware Interaction Deconfounder framework to the Inverse Propensity Scoring (IPS) estimator, incorporating personalized sensitivity bounds into robust propensity weighting for improved resilience to hidden confounding.

\textbf{PUID-DR}. PUID-DR extends the PUID framework to the Doubly Robust (DR) estimator, leveraging personalized sensitivity bounds to enhance robustness while retaining DR’s bias-reduction capability under correct specification of either the outcome or the propensity model.

\textbf{BPUID-IPS}. BPUID-IPS integrates the Benchmark-guided PUID (BPUID) framework with the IPS estimator. By combining benchmark-based regularization with personalized sensitivity bounds, BPUID-IPS balances robustness against hidden confounding with improved stability and accuracy.

\textbf{BPUID-DR}. BPUID-DR applies the BPUID framework to the DR estimator, uniting benchmark-guided training with the doubly robust formulation. This combination yields enhanced robustness to unobserved confounding while mitigating the robustness–accuracy trade-off.

\subsubsection{Evaluation Metrics}

We evaluate model performance using \textbf{UAUC} and \textbf{NDCG@K}. 

\begin{itemize}
    \item \textbf{UAUC} (User-Averaged AUC): Following~\cite{chen2021autodebias}, we compute the Area Under the ROC Curve (AUC) individually for each user who has both positive and negative feedback, and then report the average across users. This metric reflects how well the model ranks items on a \emph{per-user} basis:
    \begin{equation}
        \text{UAUC} = \frac{1}{|\mathcal{U}|} \sum_{u \in \mathcal{U}} \text{AUC}(u),
        \label{eq:uauc}
    \end{equation}
    where $\mathcal{U}$ denotes the set of users.
    
    \item \textbf{NDCG@K} (Normalized Discounted Cumulative Gain): We assess the quality of the top-$K$ recommendations by considering both relevance and position of the recommended items. The metric is defined as:
    \begin{equation}
        \text{NDCG@}K = \frac{\text{DCG@}K}{\text{IDCG@}K}, \quad
        \text{DCG@}K = \sum_{i=1}^K \frac{2^{rel_i} - 1}{\log_2(i+1)},
        \label{eq:ndcg}
    \end{equation}
    where $rel_i$ is the ground-truth relevance of the item ranked at position $i$. A higher NDCG indicates better alignment with true preferences at the top ranks. For COAT and Yahoo!R3 datasets, we use NDCG@5; for KuaiRec, we use NDCG@50.
\end{itemize}

\subsection{Performance Comparison (RQ1)}

\begin{table}[htbp]
    \centering
    \caption{Performance comparison under Q1 setting. The best result per metric within each method group is highlighted in bold.}
    \label{tab:performance-comparison-rq1}
    \begin{tabular}{lcccccc}
        \toprule
        \multirow{2}{*}{Methods} & \multicolumn{2}{c}{Coat} & \multicolumn{2}{c}{Yahoo!R3} & \multicolumn{2}{c}{KuaiRec} \\
        \cmidrule(lr){2-3} \cmidrule(lr){4-5} \cmidrule(lr){6-7}
         & UAUC & NDCG@5 & UAUC & NDCG@5 & UAUC & NDCG@50 \\
        \midrule
        MF         & 0.6606 & 0.5923 & 0.5874 & 0.5425 & 0.6781 & 0.6621 \\
        \midrule
        IPS        & 0.6627 & 0.6000 & 0.5930 & 0.5530 & 0.7446 & 0.6683 \\
        RD-IPS     & 0.6671 & 0.6081 & 0.6096 & 0.5706 & 0.7574 & 0.6700 \\
        Multi\_IPS & 0.6531 & 0.5940 & 0.5972 & 0.5589 & 0.5854 & \textbf{0.6983} \\
        PUID-IPS   & \textbf{0.6763} & \textbf{0.6182} & \textbf{0.6143} & \textbf{0.5736} & \textbf{0.7656} & 0.6703 \\
        \midrule
        DR         & 0.6797 & 0.6090 & 0.6133 & 0.5742 & 0.7468 & 0.6738 \\
        RD-DR      & 0.6827 & 0.6155 & 0.6088 & 0.5684 & 0.7635 & 0.6410 \\
        Multi\_DR  & 0.6386 & 0.5866 & 0.5671 & 0.5267 & 0.5914 & \textbf{0.6896} \\
        PUID-DR    & \textbf{0.6843} & \textbf{0.6169} & \textbf{0.6178} & \textbf{0.5753} & \textbf{0.7677} & 0.6516 \\
        \midrule
        BRD-IPS    & 0.6786 & 0.6069 & 0.6065 & 0.5655 & \textbf{0.7684} & 0.6365 \\
        BPUID-IPS  & \textbf{0.6806} & \textbf{0.6123} & \textbf{0.6125} & \textbf{0.5708} & 0.7644 & \textbf{0.6722} \\
        \midrule
        BRD-DR     & 0.6788 & 0.6107 & 0.6129 & 0.5702 & 0.7638 & 0.6361 \\
        BPUID-DR   & \textbf{0.6809} & \textbf{0.6160} & \textbf{0.6177} & \textbf{0.5745} & \textbf{0.7724} & \textbf{0.6549} \\
        \bottomrule
    \end{tabular}
\end{table}

Table~\ref{tab:performance-comparison-rq1} reports the detailed performance of PUID, BPUID, and baselines under Q1 across three datasets.

\subsubsection{Within-group comparison}

\textbf{(1) IPS-based Models.}  
PUID-IPS consistently outperforms other IPS variants on Coat and Yahoo!R3 by leveraging personalized sensitivity estimation, which adaptively reweights each user–item pair to correct hidden confounding. This avoids the over/under-adjustment problem of RD-IPS’s single global bound. In contrast, Multi\_IPS, with its multi-task shared backbone, suffers from task conflict (gradients from different debiasing subtasks interfere) and feature dilution (shared representation becomes overly generic), leading to unstable performance, especially on sparse datasets. PUID-IPS, being single-task and personalized, maintains stable generalization across both sparse and dense data.

\textbf{(2) DR-based Models.}  
PUID-DR achieves the best results on Yahoo!R3 and KuaiRec, showing that combining personalization with the DR framework mitigates propensity estimation errors and uses imputation to enhance robustness. Multi\_DR suffers from the same task conflict and feature dilution as Multi\_IPS, degrading more in sparse settings. Compared to RD-DR, PUID-DR benefits from adaptive, instance-level sensitivity bounds that capture heterogeneous confounding strength more effectively.

\textbf{(3) BPUID-based Models.}  
In the IPS framework, BPUID-IPS improves over BRD-IPS on Coat and Yahoo!R3 by adding benchmark-guided variance control to personalized sensitivity, stabilizing inverse propensity weights in sparse regimes. On KuaiRec, where data is abundant, BRD-IPS slightly surpasses BPUID-IPS, indicating reduced marginal gains from variance control. In the DR framework, BPUID-DR delivers notable gains in KuaiRec under strong confounding, preserving robustness while keeping competitive accuracy.

\subsubsection{Cross-group comparison: PUID vs.\ BPUID}

In IPS, BPUID-IPS has an edge in sparse or noisy datasets (Coat, Yahoo!R3) due to benchmark-based variance control, while PUID-IPS excels in dense data (KuaiRec) where aggressive personalization can fully leverage behavioral signals. In DR, BPUID-DR shows greater robustness under high sparsity and confounding (Yahoo!R3), whereas PUID-DR achieves higher accuracy in large-scale, heterogeneous data (KuaiRec).

\subsubsection{Overall evaluation}

PUID is well-suited for rich-interaction scenarios where per-instance sensitivity can be precisely estimated, while BPUID is more reliable in sparse or noisy environments. Both clearly outperform baselines relying on global bounds (RD) or implicit representation sharing (Multi\_IPS/Multi\_DR), validating that personalization combined with robust optimization effectively mitigates hidden confounding in recommendation.

\subsection{Robustness Under Different Levels of hidden Confounding (RQ2)}
Table~\ref{tab:maskratio-comparison} reports the performance of PUID, BPUID, RD, and BRD under the Q2 setting across three datasets, evaluated under varying hidden confounding strengths (implemented via different mask ratios).

\begin{table}[htbp]
\centering
\caption{Performance under Q2 setting. Each block is grouped by mask ratio; best scores within each block are in \textbf{bold}.}
\label{tab:maskratio-comparison}
\begin{tabular}{lcccccc}
\toprule
\multirow{2}{*}{Method} & \multicolumn{2}{c}{Coat} & \multicolumn{2}{c}{Yahoo!R3} & \multicolumn{2}{c}{KuaiRec} \\
\cmidrule(lr){2-3} \cmidrule(lr){4-5} \cmidrule(lr){6-7}
 & UAUC & NDCG@5 & UAUC & NDCG@5 & UAUC & NDCG@50 \\
\midrule
\multicolumn{7}{c}{\textbf{Mask Ratio = 0.1}} \\
\midrule
PUID-IPS   & 0.6732 & 0.5986 & 0.6174 & 0.5749 & 0.7658 & \textbf{0.6712} \\
PUID-DR    & 0.6846 & \textbf{0.6229} & 0.6155 & 0.5723 & 0.7681 & 0.6533 \\
BPUID-IPS  & 0.6786 & 0.6204 & 0.6148 & 0.5736 & 0.7650 & 0.6701 \\
BPUID-DR   & 0.6707 & 0.6091 & \textbf{0.6202} & \textbf{0.5798} & \textbf{0.7738} & 0.6523 \\
RD-IPS     & 0.6713 & 0.6098 & 0.6063 & 0.5650 & 0.7581 & 0.6704 \\
RD-DR      & \textbf{0.6861} & 0.6183 & 0.6096 & 0.5687 & 0.7629 & 0.6384 \\
BRD-IPS    & 0.6751 & 0.6154 & 0.6112 & 0.5674 & 0.7673 & 0.6405 \\
BRD-DR     & 0.6649 & 0.5939 & 0.6193 & 0.5776 & 0.7632 & 0.6322 \\
\midrule
\multicolumn{7}{c}{\textbf{Mask Ratio = 0.2}} \\
\midrule
PUID-IPS   & 0.6490 & 0.5861 & 0.6118 & 0.5702 & 0.7661 & \textbf{0.6735} \\
PUID-DR    & 0.6709 & 0.6094 & 0.6151 & 0.5702 & 0.7678 & 0.6502 \\
BPUID-IPS  & 0.6700 & 0.6122 & 0.6146 & 0.5721 & 0.7647 & 0.6683 \\
BPUID-DR   & 0.6645 & 0.5960 & \textbf{0.6197} & 0.5771 & \textbf{0.7737} & 0.6508 \\
RD-IPS     & 0.6511 & 0.5840 & 0.6098 & 0.5701 & 0.7574 & 0.6703 \\
RD-DR      & \textbf{0.6805} & \textbf{0.6196} & 0.6181 & \textbf{0.5772} & 0.7634 & 0.6389 \\
BRD-IPS    & 0.6725 & 0.6061 & 0.6066 & 0.5654 & 0.7693 & 0.6020 \\
BRD-DR     & 0.6614 & 0.5982 & 0.6149 & 0.5714 & 0.7639 & 0.6388 \\
\midrule
\multicolumn{7}{c}{\textbf{Mask Ratio = 0.5}} \\
\midrule
PUID-IPS   & 0.6214 & 0.5545 & 0.5992 & 0.5572 & 0.7648 & 0.6733 \\
PUID-DR    & \textbf{0.6512} & \textbf{0.5900} & 0.6110 & 0.5668 & 0.7676 & 0.6525 \\
BPUID-IPS  & 0.6507 & 0.5910 & 0.6064 & 0.5620 & 0.7644 & 0.6733 \\
BPUID-DR   & 0.6477 & 0.5821 & \textbf{0.6154} & \textbf{0.5710} & \textbf{0.7734} & 0.6537 \\
RD-IPS     & 0.6329 & 0.5771 & 0.6056 & 0.5634 & 0.7570 & 0.6700 \\
RD-DR      & 0.6425 & 0.5842 & 0.6032 & 0.5606 & 0.7623 & 0.6367 \\
BRD-IPS    & 0.6422 & 0.5876 & 0.6044 & 0.5602 & 0.7680 & \textbf{0.6740} \\
BRD-DR     & 0.6409 & 0.5785 & 0.6083 & 0.5633 & 0.7634 & 0.6350 \\
\midrule
\multicolumn{7}{c}{\textbf{Mask Ratio = 0.8}} \\
\midrule
PUID-IPS   & 0.6253 & 0.5463 & 0.5857 & 0.5424 & 0.7602 & 0.6624 \\
PUID-DR    & 0.6307 & 0.5580 & 0.5904 & 0.5423 & 0.7637 & 0.6385 \\
BPUID-IPS  & \textbf{0.6419} & 0.5746 & \textbf{0.6033} & \textbf{0.5624} & 0.7649 & \textbf{0.6733} \\
BPUID-DR   & 0.6352 & 0.5639 & 0.6033 & 0.5555 & \textbf{0.7694} & 0.6559 \\
RD-IPS     & 0.6270 & 0.5738 & 0.5854 & 0.5435 & 0.7518 & 0.6714 \\
RD-DR      & 0.6352 & 0.5657 & 0.5930 & 0.5470 & 0.7597 & 0.6515 \\
BRD-IPS    & 0.6399 & \textbf{0.5742} & 0.5864 & 0.5436 & 0.7686 & 0.6034 \\
BRD-DR     & 0.6403 & 0.5776 & 0.5935 & 0.5478 & 0.7585 & 0.6316 \\
\bottomrule
\end{tabular}
\end{table}

\subsubsection{Within-Group Comparison}

\paragraph{(1) IPS family.}  
PUID-IPS consistently outperforms RD-IPS and BRD-IPS when interaction data is relatively abundant. Its per-instance personalized bounds better capture heterogeneous hidden confounding, avoiding the over- or under-correction inherent in RD-IPS’s global bound. BRD-IPS gains stability via benchmark guidance but remains conservative, underutilizing fine-grained personalization. BPUID-IPS, combining personalized bounds with benchmark regularization, excels under high sparsity: the benchmark component reduces variance from inverse propensity reweighting, making corrections more reliable than PUID-IPS alone in sparse regimes.

\paragraph{(2) DR family.}  
PUID-DR achieves strong performance under moderate to low missingness by leveraging the synergy between DR and personalized bounds: DR hedges against propensity errors, while per-instance sensitivity captures residual confounding variability. RD-DR, limited by a global bound, cannot adapt to varying confounding intensities. BRD-DR gains stability from benchmark guidance, particularly in noisy or strongly confounded data, but in data-rich settings, PUID-DR’s aggressive personalization yields finer corrections. BPUID-DR balances robustness and personalization, maintaining reliability under high missingness without being overly conservative.

\subsubsection{Horizontal Comparison: PUID vs. BPUID}  
In IPS, PUID-IPS excels with sufficient observations, whereas BPUID-IPS is stronger in sparse settings due to benchmark-enhanced variance control. In DR, PUID-DR benefits from rich per-instance signals, while BPUID-DR provides stable corrections under high sparsity or confounding. Overall, PUID leverages personalization aggressively when feasible; BPUID tempers estimation uncertainty with benchmark guidance.

\subsubsection{Overall Takeaway}  
Personalized sensitivity estimation (PUID) enables fine-grained adaptation to heterogeneous hidden confounding, giving precise corrections when data is sufficient. BPUID stabilizes estimation under high missingness or noisy conditions through benchmark guidance. Both outperform global-bound baselines and naive variants, validating the effectiveness of combining personalization with robustness mechanisms.

\subsection{Experimental Validation of Entropy-Based Assumptions (RQ3, RQ4)}
To address RQ3 and RQ4, we employ the ML-100k dataset, which consists of 100,000 user–item interactions from 943 users and 1,682 items. 

To verify the validity of the proposed entropy-based assumption in quantifying hidden confounding, we design two synthetic experiments. The goal is to assess whether (1) the information gain (IG) from observed confounding and unobserved confounding exhibit consistent directional influence, and (2) the strength of hidden confounding increases as more unobserved variables are introduced.

(1) Experiment I: Proportionality Between Observed and Hidden Information Gain

In the first experiment, we investigate whether the information gain from observed confounding is proportional to that from unobserved confounding. We denote \(R(x_u)\) and \(R(x_i)\) as the proportions of observed user-side and item-side features, respectively. The features are sampled from all 16 pairwise combinations of
\[
R(x_u) \in \left\{ \frac{1}{2}, \frac{1}{3}, \frac{1}{4}, \frac{1}{5} \right\}
\quad \text{and} \quad
R(x_i) \in \left\{ \frac{1}{2}, \frac{1}{3}, \frac{1}{4}, \frac{1}{5} \right\}.
\]

For each setting, we compute the following two ratios:
\begin{equation}
\begin{aligned}
\text{Ratio}_1 &= \frac{H(o_{u,i}) - H(o_{u,i} \mid x_u)}{H(o_{u,i} \mid x_u, x_i) - H(o_{u,i} \mid x_u, x_i, h_u)} ,\\
\text{Ratio}_2 &= \frac{H(o_{u,i} \mid x_u) - H(o_{u,i} \mid x_u, x_i)}{H(o_{u,i} \mid x_u, x_i, h_u) - H(o_{u,i} \mid x_u, x_i, h_u, h_i)}.
\end{aligned}
\label{eq:entropy-ratios}
\end{equation}

Positive values of $\text{Ratio}_1$ and $\text{Ratio}_2$ suggest that the observed and hidden confounding contribute to entropy reduction in the same direction. Moreover, their magnitudes may indicate proportional relationships between the information gain from observed and unobserved variables.

Figure~\ref{fig:ig-ratios} visualizes the values of $\text{Ratio}_1$ and $\text{Ratio}_2$ across 16 combinations of user-side ($R(x_u)$) and item-side ($R(x_i)$) feature observability using heatmaps. We observe that all values are positive, indicating that both observed and hidden confounding reduce entropy in the same direction. This consistent sign across all settings provides strong empirical support for the assumption that observed and hidden variables contribute aligned information gains.

\begin{figure}[htbp]
\centering
\includegraphics[width=0.8\linewidth]{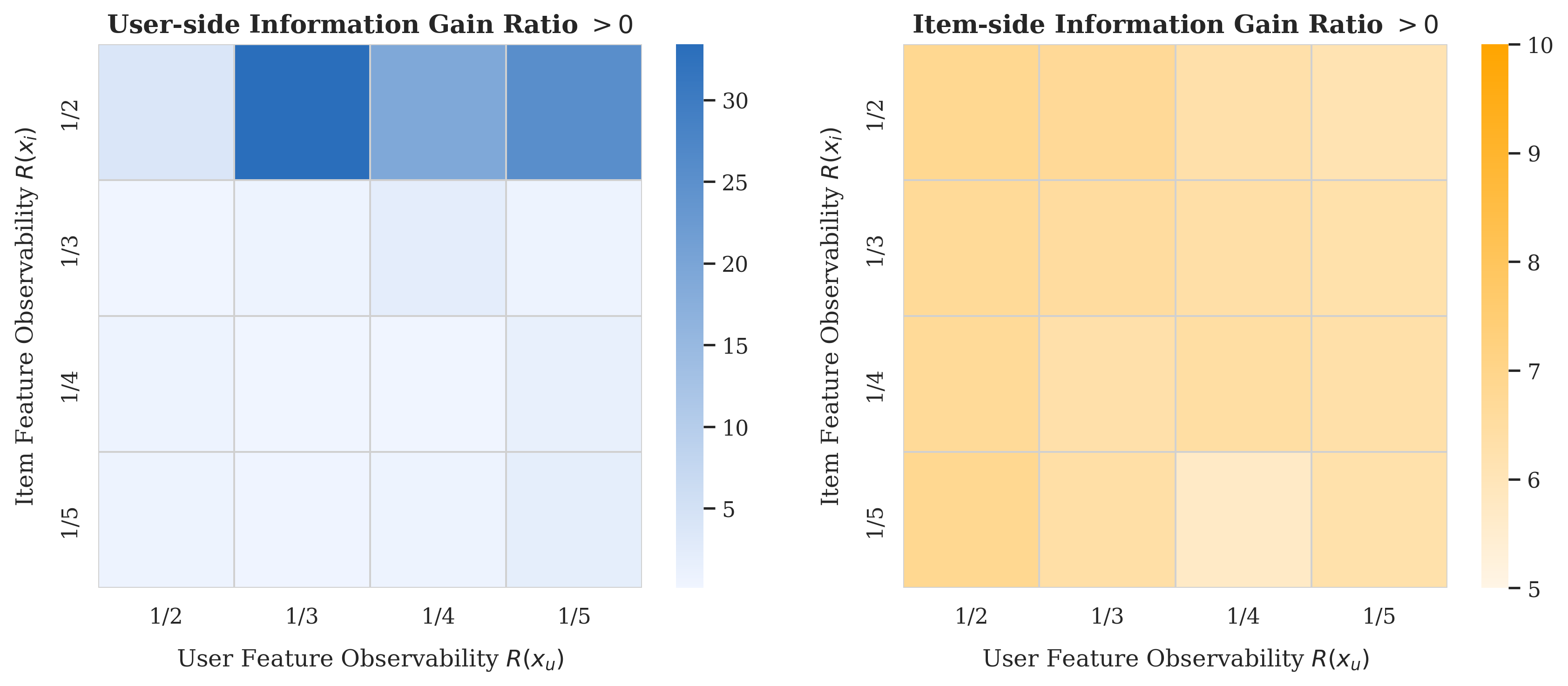}
\caption{
Heatmaps of information gain ratios under varying user-side ($R(x_u)$) and item-side ($R(x_i)$) feature observability. 
\textbf{(Left)} User-side ratio values (0.12--33.42) scaled to 5--10 range, showing relative contribution of observed user features. 
\textbf{(Right)} Item-side log-scaled ratios (5--10 in $\log_{10}$ scale), indicating magnitude of item feature contributions. 
Warmer colors represent stronger effects, with consistent positive values across all configurations suggesting stable entropy reduction from both observed and unobserved confounding.The consistently positive pattern across all combinations suggests a stable and directionally aligned contribution of observed and unobserved confounding to entropy reduction.
}
\label{fig:ig-ratios}
\end{figure}

(2) Experiment II: Monotonicity of Hidden Confounding Strength

The second experiment investigates whether the strength of hidden confounding increases as the number of unobserved variables grows. We design two sub-experiments:

\textbf{Experiment II-a:} Fix all item-side features as observed, and gradually increase the number of unobserved user features. We compute the entropy reduction:
\begin{equation}
\Delta H_{\text{user}} = H(o_{u,i} \mid x_u, x_i) - H(o_{u,i} \mid x_u, x_i, h_u).
\label{eq:delta-h-user}
\end{equation}

\textbf{Experiment II-b:} Fix all user-side features as observed, and gradually increase the number of unobserved item features. We compute the entropy reduction:
\begin{equation}
\Delta H_{\text{item}} = H(o_{u,i} \mid x_u, x_i, h_u) - H(o_{u,i} \mid x_u, x_i, h_u, h_i).
\label{eq:delta-h-item}
\end{equation}

Figure~\ref{fig:gain-hu-hi} shows that as the number of unobserved user features increases, the corresponding information gain steadily grows, indicating a monotonic effect of user-side hidden confounding on predictive uncertainty.
Similarly, increasing the number of unobserved item features leads to a consistent rise in information gain, confirming that item-side hidden confounding also exert greater influence as their dimensionality expands.

\begin{figure}[htbp]
\centering
\includegraphics[width=0.85\linewidth]{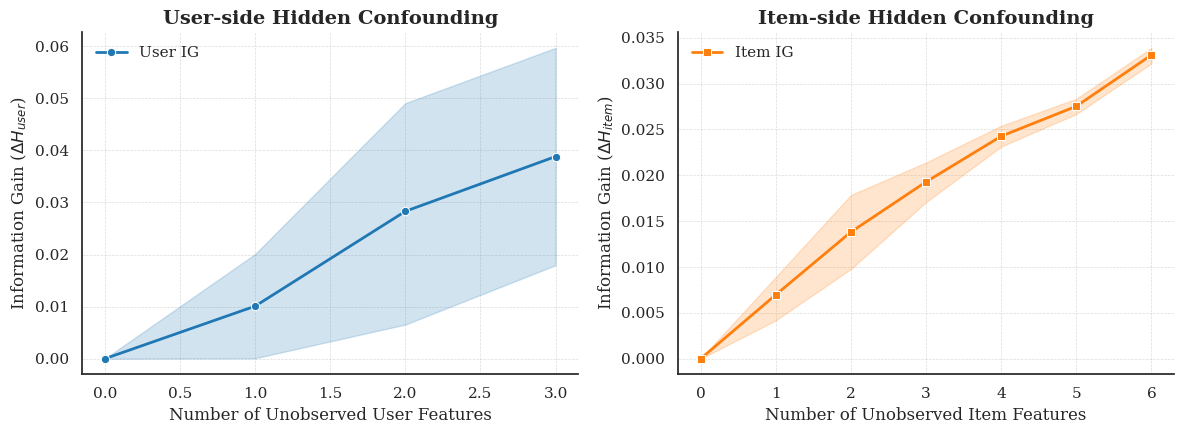}
\caption{The information gain from unobserved confounding increases consistently with the number of hidden user features (\textbf{left}) and item features (\textbf{right}). This confirms the monotonic influence of hidden confounding on the prediction uncertainty, providing empirical evidence for modeling hidden confounding strength via feature cardinality.}
\label{fig:gain-hu-hi}
\end{figure}

(3) Summary

The above two experiments provide empirical evidence to support our entropy-based assumption: (1) observed and unobserved confounding affect exposure entropy in a directionally consistent and potentially proportional manner, and (2) the strength of hidden confounding increases with the number of hidden variables. These findings establish a solid foundation for estimating personalized sensitivity bounds $\Gamma_{u,i}$ via entropy gap measures.

\subsection{Ablation Study}
To better understand the contributions of key components in our framework, we conduct systematic ablation studies on the Coat dataset, focusing on three aspects:

\begin{itemize}
\item \textbf{Feature-aware Representation}: Comparing models with side features versus ID-only embeddings.
\item \textbf{Personalized Bounds}: Comparing personalized bounds versus global bounds within IPS.
\item \textbf{Sensitivity to $\alpha, \beta$}: Examining how the personalized bounds, modulated by coefficients $\alpha$ and $\beta$, influence performance.
\end{itemize}

\paragraph{Results on IPS-based Models.}
Figure~\ref{fig:ablation} (left) reports results for IPS-based models. Somewhat unexpectedly, removing side features (PUID-IPS without feature) yields better UAUC and NDCG@5, suggesting that for the small and sparse Coat dataset, side information may introduce noise rather than benefits. Furthermore, different personalized bounds leads to different performance. This highlights the importance of tuning the sensitivity bounds to adapt to dataset characteristics. The larger $\alpha, \beta$ make the model more tolerant to potential exposure bias, which appears beneficial in sparse settings like Coat.

\paragraph{Results on DR-based Models.}
Figure~\ref{fig:ablation} (right) presents DR-based results. The performance difference between using or omitting side features is marginal, suggesting DR is inherently more robust to sparse or noisy side information. Moreover, varying $(\alpha, \beta)$ has negligible impact on DR performance, in contrast to IPS. The RD-DR baseline performs similarly, further confirming DR's relative insensitivity to personalized bounds in this case.

\vspace{0.5em}
\noindent\textbf{Overall Findings:} These results suggest that (i) ID-only representations remain competitive under data sparsity; (ii) personalized bounds controlled by $(\alpha, \beta)$ help mitigate bias, especially for IPS; and (iii) the sensitivity to $(\alpha, \beta)$ is more pronounced for IPS but less critical for DR.


\begin{figure}[htbp]
\centering
\includegraphics[width=0.9\linewidth]{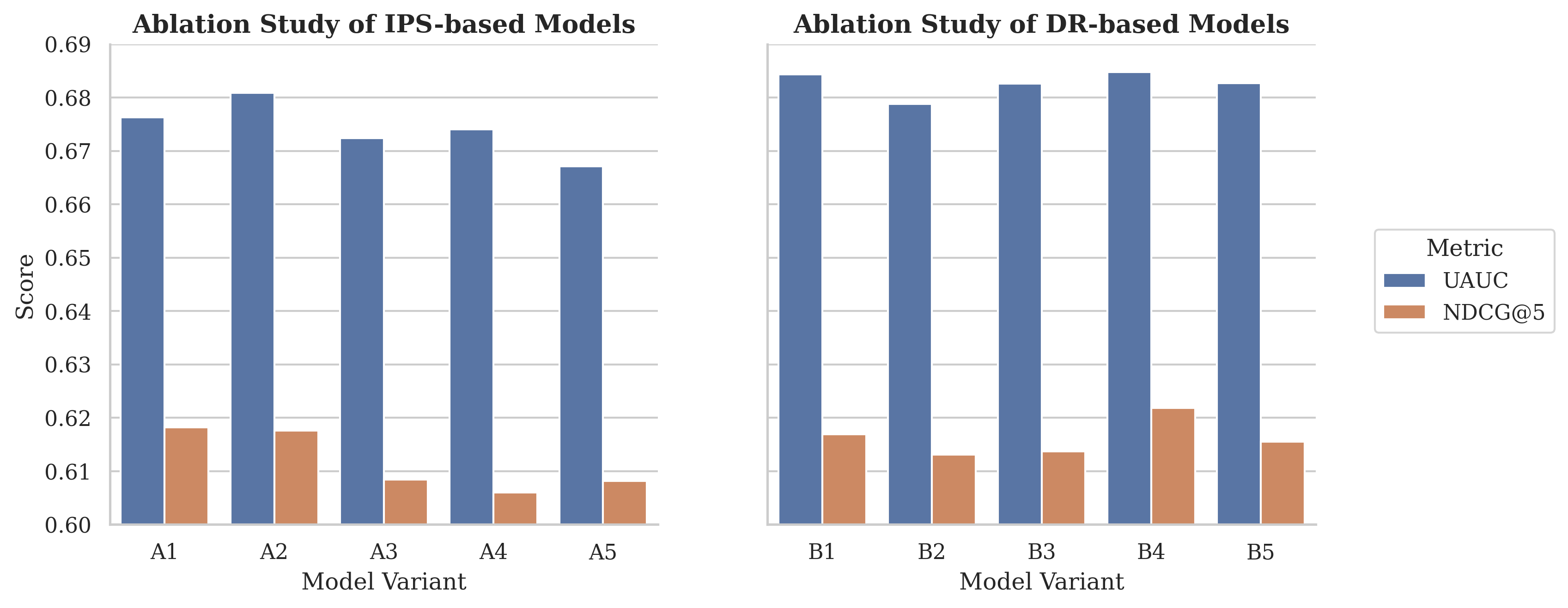}
\caption{
Comparative analysis of model variants under the IPS (A-series) and DR (B-series) frameworks on the Coat dataset. The IPS-based models comprise: A1 (baseline PUID-IPS with side features, $\alpha=2,\beta=5$), A2 (PUID-IPS without side features), A3 (PUID-IPS with $\alpha=5,\beta=10$), A4 (PUID-IPS with $\alpha=5,\beta=5$), and A5 (RD-IPS with random bounds as baseline). The DR-based models include: B1 (baseline PUID-DR with side features, $\alpha=2,\beta=5$), B2 (PUID-DR without side features), B3 (PUID-DR with $\alpha=5,\beta=10$), B4 (PUID-DR with $\alpha=5,\beta=5$), and B5 (RD-DR with random bounds as baseline). This study examines three key aspects: (1) \textbf{feature-aware representation} (comparing A1 vs A2 and B1 vs B2), (2) \textbf{degree of personalization} (contrasting PUID variants against RD baselines), and (3) \textbf{boundary sensitivity} (evaluating different $(\alpha,\beta)$ configurations). Results demonstrate that personalized bounds consistently outperform global bounds, the impact of side features varies significantly between frameworks, and careful tuning of $(\alpha,\beta)$ parameters proves critical for optimal performance.
}
\label{fig:ablation}
\end{figure}

\section{Conclusion}
In this work, we address the challenge of hidden confounding in recommender systems, where standard propensity-based methods, such as IPW and doubly robust estimators, are biased because the true propensities cannot be accurately estimated. Existing approaches, including Robust Deconfounder (RD), leverage sensitivity analysis to estimate an uncertainty set of the true propensity, but they assume homogeneous confounding strength across all user–item pairs, which fails to capture heterogeneity in real-world data. To overcome this limitation, we propose \textbf{PUID}, which estimates personalized sensitivity bounds at the user–item level using an entropy-based information gain measure, enabling the model to adapt to varying degrees of hidden confounding. We assume that the effect of unobserved confounders can be reflected by the portion of exposure uncertainty not explained by observed features. Building on this, we develop \textbf{BPUID}, a benchmark-guided variant that incorporates pre-trained models to stabilize learning while maintaining robustness and predictive accuracy. Extensive experiments on three real-world datasets demonstrate that PUID and BPUID consistently outperform state-of-the-art baselines under hidden confounding. Future work includes improving the estimation of personalized information gains, extending the framework to sequential or context-aware recommendation, and exploring automatic tuning of sensitivity hyperparameters across different scenarios.

\section{Acknowledgements}

I would like to express my sincere gratitude to Li Haoxuan, Su Wanting, Xiao Yanghao, Zheng Chunyuan, Li Xiang, and Xia Tianyu for their valuable contributions, engaging discussions, and continuous support throughout this research. I am grateful for their collaboration and insights that helped shape this work. I would also like to clarify that this study is intended to be incorporated as part of my graduation thesis at Guangdong University of Technology, and I do not plan to pursue further publication of this specific research beyond its use for my degree. All names are listed here in no particular order (names have been shuffled), and I sincerely appreciate each of their generous help.

\bibliography{ref}

\section{Appendix}
\appendix
\label{appendix:feature_processing}
\section{Feature Processing Details}

\subsection{a. \textbf{Feature Selection Based on Correlation}}
To mitigate collinearity caused by feature redundancy during model training, we adopt a Pearson correlation-based feature selection strategy prior to discretization. This approach calculates correlation coefficients between numerical features and groups them by similarity. Only one representative feature from each group is retained, thereby simplifying the input space and improving model stability and generalization capability.

For \textbf{item-side features}, any feature group with a Pearson correlation coefficient exceeding 0.9 is considered highly redundant. One representative feature from each such group is preserved. The final high-correlation groups are listed below:
\begin{itemize}
\item Play Behavior Group: \{show\_cnt, show\_user\_num, play\_cnt, play\_user\_num, complete\_play\_cnt, complete\_play\_user\_num, valid\_play\_cnt, valid\_play\_user\_num, short\_time\_play\_cnt, short\_time\_play\_user\_num, 
long\_time\_play\_cnt, 

long\_time\_play\_user\_num\}
\item Like Behavior Group: \{like\_cnt, like\_user\_num, double\_click\_cnt, click\_like\_cnt\}
\item Cancel Like Group: \{cancel\_like\_cnt, cancel\_like\_user\_num\}
\item Comment Group: \{comment\_cnt, comment\_user\_num, direct\_comment\_cnt\}
\item Delete Comment Group: \{delete\_comment\_cnt, delete\_comment\_user\_num\}
\item Comment Like Group: \{comment\_like\_cnt, comment\_like\_user\_num\}
\item Follow Group: \{follow\_cnt, follow\_user\_num\}
\item Cancel Follow Group: \{cancel\_follow\_cnt, cancel\_follow\_user\_num\}
\item Share Group: \{share\_cnt, share\_user\_num\}
\item Download Group: \{download\_cnt, download\_user\_num\}
\item Report Group: \{report\_cnt, report\_user\_num\}
\item Reduce Recommendation Group: \{reduce\_similar\_cnt, reduce\_similar\_user\_num\}
\item Collect Group: \{collect\_cnt, collect\_user\_num\}
\item Cancel Collect Group: \{cancel\_collect\_cnt, cancel\_collect\_user\_num\}
\end{itemize}

From each group, only a single feature is retained for downstream modeling. For example, \texttt{valid\_play\_user\_num} is selected from the Play Behavior Group.

For \textbf{user-side features}, only one highly correlated pair is detected: \{onehot\_feat2, onehot\_feat3\} (Pearson correlation $> 0.7$).

The final selected features are as follows:
\begin{itemize}
\item \textbf{Item-side}: video\_id, video\_type, upload\_type, visible\_status, video\_tag\_name, valid\_play\_user\_num, play\_progress, click\_like\_cnt, cancel\_like\_user\_num, direct\_comment\_cnt, comment\_like\_user\_num, follow\_cnt, cancel\_follow\_user\_num, share\_cnt, download\_user\_num, report\_cnt, reduce\_similar\_user\_num, collect\_cnt, cancel\_collect\_user\_num
\item \textbf{User-side}: user\_active\_degree, is\_live\_streamer, follow\_user\_num\_range, friend\_user\_num\_range, register\_days\_range, onehot\_feat0, onehot\_feat6, onehot\_feat11, onehot\_feat12, onehot\_feat13, onehot\_feat14, onehot\_feat15, onehot\_feat16, onehot\_feat17
\end{itemize}

This feature selection mechanism reduces the dimensionality of the feature space, minimizes redundancy, and mitigates overfitting, thereby providing a more compact and stable input foundation for model training.

\subsection{b. \textbf{Discretization of Features}}
To standardize numerical distributions and enhance model adaptability to heterogeneous inputs, we discretize both categorical and numerical features on the user and item sides. The general strategy is as follows: non-numeric features are processed using manual or grouped mappings, while numeric features are primarily binned using \texttt{qcut} based on quantiles. In the presence of extreme outliers or dense zero values, we apply conditional binarization or fall back to equal-width binning. All features are ultimately converted to integer form to support embedding operations during training.

\textbf{Part 1: Feature-wise Processing of \texttt{item features}} \\
\textit{Note: Binning summary \{m: n\} indicates that group label $m$ contains $n$ samples.}
\begin{itemize}
\item \texttt{video\_id}: Identifier reserved for indexing or joins; no discretization applied.
\item \texttt{video\_type}: 'AD' $\rightarrow$ 0; others $\rightarrow$ 1. Binning: \{1: 342188, 0: 1153\}.
\item \texttt{upload\_type}: 'ShortImport' $\rightarrow$ 0; others $\rightarrow$ 1. Binning: \{0: 201962, 1: 141379\}.
\item \texttt{visible\_status}: Ordered categorical: 'only friends' $\rightarrow$ 0, 'private' $\rightarrow$ 1, 'public' $\rightarrow$ 2. Binning: \{2: 340782, 1: 2515, 0: 44\}.
\item \texttt{video\_tag\_name}: High cardinality; top tag $\rightarrow$ bin 0, others binned via \texttt{qcut} into 4 groups, NaNs $\rightarrow$ -1, bins 0--3 merged into category 1. Binning: \{4.0: 232329, -1.0: 32434, 1.0: 78578\}.
\item \texttt{valid\_play\_user\_num}: 0 isolated, others \texttt{qcut}-binned into 3. Adjusted: bin 3 $\rightarrow$ 2, bin 0 $\rightarrow$ 1. Final bins: \{2: 171598, 1: 171743\}.
\item \texttt{play\_progress}: \texttt{qcut} binned into \{3,2,1,0\}, then remapped 3 $\rightarrow$ 2, 0 $\rightarrow$ 1. Final bins: \{2: 171670, 1: 171671\}.
\item \texttt{click\_like\_cnt}: \texttt{qcut}, bin 2 $\rightarrow$ 1. Final: \{1: 171024, 0: 172317\}.
\item \texttt{cancel\_like\_user\_num}: \texttt{qcut}, bin 2 $\rightarrow$ 1. Final: \{1: 171669, 0: 171672\}.
\item \texttt{direct\_comment\_cnt}: Binarized: 0 vs. non-zero. \{1: 140857, 0: 202484\}.
\item \texttt{comment\_like\_user\_num}: Binarized. \{1: 150595, 0: 192746\}.
\item \texttt{follow\_cnt}: Binarized. \{1: 179610, 0: 163731\}.
\item \texttt{cancel\_follow\_user\_num}: Binarized. \{0: 338216, 1: 5125\}.
\item \texttt{share\_cnt}: Binarized. \{1: 124567, 0: 218774\}.
\item \texttt{download\_user\_num}: Binarized. \{1: 127294, 0: 216047\}.
\item \texttt{report\_cnt}: Binarized. \{0: 337041, 1: 6300\}.
\item \texttt{reduce\_similar\_user\_num}: Binarized. \{1: 148913, 0: 194428\}.
\item \texttt{collect\_cnt}: Binarized. \{1: 152523, 0: 190818\}.
\item \texttt{cancel\_collect\_user\_num}: Binarized. \{1: 130345, 0: 212996\}.
\end{itemize}

\textbf{Part 2: Feature-wise Processing of \texttt{user features}} \\
\textit{Note: \{m: n\} indicates that category $m$ contains $n$ samples.}
\begin{itemize}
\item \texttt{user\_active\_degree}: 'full\_active' $\rightarrow$ 1, 'high\_active' / 'middle\_active' $\rightarrow$ 2, 'UNKNOWN' $\rightarrow$ 0. \{2: 966, 1: 6092, 0: 118\}.
\item \texttt{is\_live\_streamer}: Boolean; used as-is. \{0: 7127, -1: 49\}.
\item \texttt{follow\_user\_num\_range}: '0','(0,10]','(10,50]' $\rightarrow$ 0; others $\rightarrow$ 1. \{0: 4155, 1: 3021\}.
\item \texttt{friend\_user\_num\_range}: '0' $\rightarrow$ 0; others $\rightarrow$ 1. \{0: 4338, 1: 2838\}.
\item \texttt{register\_days\_range}: '15--30','31--60','61--90','91--180' $\rightarrow$ 0; others $\rightarrow$ 1. \{0: 2936, 1: 4240\}.
\item \texttt{onehot\_feat0}, \texttt{onehot\_feat6}, \texttt{onehot\_feat11}--\texttt{onehot\_feat17}: Encrypted one-hot features; retained as-is. Example distributions: \{0: 4361, 1: 2815\}, \{1: 4040, 0: 3120, 2: 16\}, \{0: 6665, 2: 473, 1: 38\}, etc.
\end{itemize}

\section{Theoretical Foundations of BPUID Framework}

\subsection{Theoretical Justification for Benchmark Extension}

\textbf{Assumption 1 (Lipschitz Continuity of Confounding Bias).} 
The bias function $B(\phi) = \mathbb{E}[\hat{R}(\phi)] - R(\phi)$ satisfies:
\[
|B(\phi_1) - B(\phi_2)| \leq L \cdot \|\phi_1 - \phi_2\|,
\]
for some Lipschitz constant $L > 0$ and appropriate parameter norm $\|\cdot\|$.

\textbf{Theorem 1 (Incremental Estimation Robustness).} 
Under Assumption 1, for any benchmark model $\phi^{(0)}$, the estimation error of the incremental risk $\Delta R(\phi,\phi^{(0)}) = R(\phi) - R(\phi^{(0)})$ satisfies:
\[
|\mathbb{E}[\Delta \hat{R}(\phi)] - \Delta R(\phi)| \leq L \cdot \|\phi - \phi^{(0)}\|,
\]
where $\Delta \hat{R}(\phi) = \hat{R}(\phi) - \hat{R}(\phi^{(0)})$.

\textbf{Proof.} 
\begin{align*}
\mathbb{E}[\Delta \hat{R}(\phi)] - \Delta R(\phi) 
&= \left(\mathbb{E}[\hat{R}(\phi)] - R(\phi)\right) - \left(\mathbb{E}[\hat{R}(\phi^{(0)})] - R(\phi^{(0)})\right) \\
&= B(\phi) - B(\phi^{(0)}) \\
&\leq L \cdot \|\phi - \phi^{(0)}\|.
\end{align*}
\hfill $\square$

\subsection{Confounding Adjustment with Tighter Bounds}

\textbf{Assumption 2 (Rosenbaum Sensitivity Model).} 
There exists $\Gamma \geq 1$ such that for all $(u,i)$ with the same observed features $X$:
\[
\frac{1}{\Gamma} \leq \frac{\pi_{u,i}^*/(1-\pi_{u,i}^*)}{\hat{\pi}_{u,i}/(1-\hat{\pi}_{u,i})} \leq \Gamma,
\]
where $\pi_{u,i}^* = P(O_{u,i}=1|X,H)$ is the true propensity and $\hat{\pi}_{u,i} = P(O_{u,i}=1|X)$ is the estimated propensity.

\textbf{Theorem 2 (Tighter Bias Bound under Sensitivity Analysis).} 
Under Assumption 2, the worst-case bias of the BPUID estimator satisfies:
\[
|B_{\text{BPUID}}(\phi)| \leq C \cdot (\Gamma-1) \cdot \|\Delta e\|_\infty,
\]
while the standard PUID estimator satisfies:
\[
|B_{\text{PUID}}(\phi)| \leq C \cdot \Gamma \cdot \|e(\phi)\|_\infty,
\]
where $C$ is a constant depending only on the data distribution, and $\Delta e = e(\phi) - e(\phi^{(0)})$.

\textbf{Proof.} 
For any weight function $w^* = \pi^*/\hat{\pi}$ satisfying Assumption 2, we have:
\begin{align*}
|B_{\text{BPUID}}(\phi)| &= |\mathbb{E}[w^*\Delta e] - \mathbb{E}[\Delta e]| \\
&= |\mathbb{E}[(w^*-1)\Delta e]| \\
&\leq \mathbb{E}[|w^*-1| \cdot |\Delta e|] \\
&\leq (\Gamma-1) \cdot \mathbb{E}[|\Delta e|] \\
&\leq C \cdot (\Gamma-1) \cdot \|\Delta e\|_\infty.
\end{align*}
For the standard PUID estimator, the same analysis yields $|B_{\text{PUID}}(\phi)| \leq C \cdot \Gamma \cdot \|e(\phi)\|_\infty$ with $\Delta e$ replaced by $e(\phi)$. \hfill $\square$

\subsection{Guarantee of Accuracy Improvement}

\textbf{Assumption 3 (Bounded Loss Variance).} 
The loss function satisfies $\text{Var}(\mathcal{L}(\phi)) \leq \sigma^2$ for all $\phi$ in the feasible parameter space.

\textbf{Theorem 3 (Generalization Bound).} 
Under Assumptions 1 and 3, let $\phi^*$ minimize $\mathcal{L}_{\text{BPUID}}(\phi)$ with $\mathcal{L}_{\text{BPUID}}(\phi^*) = -\delta < 0$. Then with probability at least $1-\eta$:
\[
R(\phi^*) \leq R(\phi^{(0)}) + L \|\phi^* - \phi^{(0)}\| - \delta + \sqrt{\frac{2\sigma^2 \log(2/\eta)}{n}}.
\]

\textbf{Proof.} 
Define $\epsilon(\phi) = \mathcal{L}(\phi) - \mathbb{E}[\mathcal{L}(\phi)]$. Then:
\begin{align*}
R(\phi^*) - R(\phi^{(0)}) 
&= [\mathcal{L}(\phi^*) - \mathcal{L}(\phi^{(0)})] + [B(\phi^*) - B(\phi^{(0)})] + [\epsilon(\phi^{(0)}) - \epsilon(\phi^*)] \\
&\leq -\delta + L\|\phi^* - \phi^{(0)}\| + |\epsilon(\phi^{(0)}) - \epsilon(\phi^*)|.
\end{align*}
By Hoeffding's inequality, with probability at least $1-\eta$:
\[
|\epsilon(\phi^{(0)}) - \epsilon(\phi^*)| \leq \sqrt{\frac{2\sigma^2 \log(2/\eta)}{n}}.
\]
Combining these inequalities yields the theorem statement. \hfill $\square$

\textbf{Corollary 3.1 (Accuracy Improvement Condition).} 
From Theorem 3, $R(\phi^*) < R(\phi^{(0)})$ holds when:
\[
\delta > L \|\phi^* - \phi^{(0)}\| + \sqrt{\frac{2\sigma^2 \log(2/\eta)}{n}}.
\]

\textbf{Proof.} 
From Theorem 3, $R(\phi^*) < R(\phi^{(0)})$ requires:
\[
L \|\phi^* - \phi^{(0)}\| - \delta + \sqrt{\frac{2\sigma^2 \log(2/\eta)}{n}} < 0,
\]
which rearranges to the stated condition. \hfill $\square$

\end{document}